\documentclass{article}

\usepackage{amssymb}
\usepackage{setspace}
\usepackage{amsmath}
\usepackage{amsthm}
\usepackage{amsfonts}
\usepackage{color}
\usepackage{array}
\usepackage{graphicx}
\graphicspath{ {./figs/} }
\usepackage{wrapfig}
\usepackage{mathtools}
\usepackage[ruled,linesnumbered]{algorithm2e}
\usepackage{multirow}
\usepackage{chngcntr}
\usepackage{apptools}
\usepackage{multicol}
\usepackage{float}
\usepackage{caption}
\usepackage{subcaption}
\usepackage{chngcntr}
\usepackage[hang,flushmargin]{footmisc}

\usepackage[final]{corl_2019} 

\usepackage{enumitem}
\newtheorem{theorem}{Theorem}
\newtheorem{problem}{Problem}
\newtheorem{definition}{Definition}
\newtheorem{corollary}{Corollary}
\newtheorem{lemma}{Lemma}
\AtAppendix{\counterwithin{theorem}{section}}
\AtAppendix{\counterwithin{corollary}{section}}
\AtAppendix{\counterwithin{definition}{section}}
\AtAppendix{\counterwithin{lemma}{section}}
\AtAppendix{\counterwithin{figure}{section}}

\newcommand{\task}{\Pi}

\newcommand{\traj}{\xi}
\newcommand{\state}{x}
\newcommand{\statespace}{\mathcal{X}}
\newcommand{\safeset}{\mathcal{S}}
\newcommand{\unsafeset}{\mathcal{A}}

\newcommand{\umax}{\Delta\state}
\newcommand{\sd}{\textsf{sd}}

\newcommand{\numsafe}{N_s}
\newcommand{\numunsafe}{N_{\neg s}}
\newcommand{\control}{u}
\newcommand{\controlset}{\mathcal{U}}
\newcommand{\trajset}{\mathcal{T}}
\newcommand{\constraintset}{\mathcal{C}}
\newcommand{\unsafetrajset}{\mathcal{T}_\mathcal{A}}

\newcommand{\trajxu}{\traj_{xu}}
\newcommand{\trajx}{\traj_\state}
\newcommand{\traju}{\traj_\control}

\newcommand{\constraintspace}{\mathcal{C}}

\newcommand{\guarunsafe}{\mathcal{G}_{\neg s}}
\newcommand{\guarsafe}{\mathcal{G}_s}
\newcommand{\feas}{\mathcal{F}}
\newcommand{\cstate}{k}
\newcommand{\numbox}{N^*}

\theoremstyle{remark}
\newtheorem*{rem}{Remark}

\title{Learning Parametric Constraints in High Dimensions from Demonstrations}

\author{
  Glen Chou, Necmiye Ozay, and Dmitry Berenson\\
  Department of Electrical Engineering and Computer Science\\
  University of Michigan, Ann Arbor\\
  \texttt{\{gchou, necmiye, dmitryb\}@umich.edu} \\
}

\begin{document}
\maketitle

\vspace{-25pt}
\begin{abstract}
We present a scalable algorithm for learning parametric constraints in high dimensions from safe expert demonstrations. To reduce the ill-posedness of the constraint recovery problem, our method uses hit-and-run sampling to generate lower cost, and thus unsafe, trajectories. Both safe and unsafe trajectories are used to obtain a representation of the unsafe set that is compatible with the data by solving an integer program in that representation's parameter space. Our method can either leverage a known parameterization or incrementally grow a parameterization while remaining consistent with the data, and we provide theoretical guarantees on the conservativeness of the recovered unsafe set. We evaluate our method on high-dimensional constraints for high-dimensional systems by learning constraints for 7-DOF arm, quadrotor, and planar pushing examples, and show that our method outperforms baseline approaches.
\end{abstract}


\vspace{-12pt}
\section{Introduction}
\vspace{-7pt}
Learning from demonstration is a powerful paradigm for enabling robots to perform complex tasks. Inverse optimal control and inverse reinforcement learning (IOC/IRL) (\cite{irl_1, irl_2, lfd3, ng_irl}) methods have been used to learn a cost function to replicate the behavior of an expert demonstrator. However, planning problems generally also require knowledge of constraints, which define the states or trajectories that are safe. For example, to get a robot arm to efficiently transport a cup of coffee without spilling it, one can optimize a cost function describing the length of the path, subject to constraints on the pose of the end effector. Constraints can represent safety requirements more strictly than cost functions, especially in safety-critical situations: enforcing a hard constraint can enable the robot to guarantee safe behavior, as opposed to using a ``softened" cost penalty term. Furthermore, learning a global constraint shared across many tasks can help the robot generalize. Consider the arm, which must avoid spilling the coffee regardless of where the cup started off or needs to go.

While constraints are important, it can be impractical to exhaustively program all the possible constraints a robot should obey across all tasks. Thus, we consider the problem of extracting the latent constraints within expert demonstrations that are shared across tasks. We adopt the insight of \cite{extended_version} that each safe, optimal demonstration induces a set of lower-cost trajectories that must be unsafe due to violation of an unknown constraint. As in \cite{extended_version}, we sample these unsafe trajectories, ensuring that they are also consistent with the system dynamics, control constraints, and start/goal constraints. The unsafe trajectories are used together with the safe demonstrations in an ``inverse" integer program that recovers an unsafe set consistent with the safe and unsafe trajectories. We make the following additional contributions in this paper. First, by using a (potentially known) parameterization of the constraints, our method enables the inference of safe and unsafe sets in high-dimensional constraint spaces. Second, we relax the known parametrization assumption and propose a means to incrementally grow a parameterization with the data. Third, we introduce a method for extracting volumes of states which are guaranteed safe or guaranteed unsafe according to the data and parameterization. Fourth, we provide theoretical analysis showing that our method is guaranteed to output conservative estimates of the unsafe and safe sets under mild assumptions. Finally, we evaluate our method on high-dimensional constraints for high-dimensional systems by learning constraints for 7-DOF arm, quadrotor, and planar pushing examples, showing that our method outperforms baseline approaches.

\vspace{-12pt}
\section{Related Work}
\vspace{-8pt}
Inverse optimal control \cite{kalman, boyd} (IOC) and inverse reinforcement learning (IRL) \cite{ng_irl} aim to recover an objective function that replicates provided expert demonstrations when optimized. Our method is complementary to these approaches; if the demonstrator solves a constrained optimization problem, we are finding its constraints, given the cost; IOC/IRL finds the cost, given the constraints \cite{toussaint}. Risk-sensitive IRL \citep{sumeet} is complementary to our work, which learns hard constraints. Similarly, \cite{satinder} learns a state-space constraint shared across tasks as a penalty term in the reward function of an MDP. However, when representing a constraint as a penalty, it is unclear if a demonstrated action was made to avoid a penalty or to improve the trajectory cost in terms of the true cost function (or both). Thus, learning a penalty generalizing across cost functions becomes difficult. To avoid this, we assume a known cost function to explicitly reason about the constraint. Also relevant is safe reinforcement learning, which aims to perform exploration while minimizing visitation of unsafe states. Several methods \cite{safe_exploration, krause, claire} use Gaussian process models to incrementally explore safe regions in the state space. We take a complementary approach to safe learning by using demonstrations in place of exploration to guide the learning of safe behaviors. Methods exist for learning geometric state space constraints \citep{vijayakumar, shah}, task space equality constraints \citep{howard1, howard2}, and convex constraints \citep{melanie}, which our algorithm generalizes by being able to learn arbitrary nonconvex parametric inequality constraints defined in some constraint space (not limited to the state space). Other methods aim to learn local trajectory-based constraints \cite{dmitry, anca, lfdc1,lfdc2,lfdc3,lfdc4} by reasoning over the constraints within a single trajectory or task. In contrast, our method aims to learn a global constraint shared across tasks. 

The method closest to our work is \cite{extended_version}, which learns a global shared constraint on a gridded constraint space; hence, the resulting constraint recovery method scales exponentially with the constraint space dimension and cannot exploit any side information on the structure of the constraint. This often leads to very conservative estimates of the unsafe set, and only grid cells visited by demonstrations can be learned guaranteed safe. We overcome these shortcomings with a novel algorithm that exploits constraint parameterizations for scalability and integration of prior knowledge, and also enables learning volumes of guaranteed safe/unsafe states in the original non-discretized constraint space, yielding less conservative estimates of the safe/unsafe sets under weaker assumptions than \cite{extended_version}.

\vspace{-8pt}
\section{Problem Setup}
\vspace{-8pt}
Consider a system with discrete-time dynamics $\state_{t+1} = f(\state_t, \control_t, t)$ or continuous-time dynamics $\dot\state = f(\state, \control, t)$, where $\state\in\statespace$ and $\control\in\controlset$. The system performs tasks $\task$ represented as constrained optimization problems over state/control trajectories $\trajx$/$\traju$ in state/control trajectory space $\trajset^\state$/$\trajset^\control$:
\begin{problem}[Forward problem / ``task" $\task$]\label{prob:fwd_prob}
\vspace{-3pt}
\begin{equation}\label{eq:fwdprob}
	\begin{array}{>{\displaystyle}c >{\displaystyle}l >{\displaystyle}l}
		\underset{\trajx, \traju}{\textnormal{min}} & \quad c_\task(\trajx, \traju) &\\
		\textnormal{s.t.} & \quad \phi(\trajx, \traju) \in \safeset(\theta) \subseteq \constraintspace\\
		& \quad \bar\phi(\trajx, \traju) \in \bar\safeset \subseteq \bar\constraintspace\\
		& \quad \phi_\task(\trajx, \traju) \in \safeset_\task \subseteq \constraintspace_\task\\
	\end{array}\hspace{-15pt}
\end{equation}
\end{problem}
\vspace{-8pt}
\noindent where $c_\task(\cdot): \trajset^\state \times \trajset^\control \rightarrow \mathbb{R}$ is a cost function for task $\task$, and $\phi(\cdot, \cdot) : \trajset^\state \times \trajset^\control \rightarrow \constraintspace$ is a known mapping from state-control trajectories to a constraint space $\constraintspace$, elements of which are referred to as ``constraint states". Mappings $\bar\phi(\cdot,\cdot): \trajset^\state \times \trajset^\control \rightarrow \bar\constraintspace$ and $\phi_\task(\cdot, \cdot): \trajset^\state \times \trajset^\control \rightarrow \constraintspace_\task$ are known and map to constraint spaces $\bar\constraintspace$ and  $\constraintspace_\task$, containing a known shared safe set $\bar \safeset$ and a known task-dependent safe set $\safeset_\task$, respectively. In this paper, we take $\trajset_{\safeset_\task}$ to be the set of trajectories satisfying start/goal state constraints and $\trajset_{\bar\safeset}$ to be the set of dynamically-feasible trajectories obeying control constraints, though the dynamics may not be known in closed form. $\safeset(\theta) = \{\cstate \in \constraintspace \mid g(\cstate, \theta) > 0 \}$ is an unknown safe set defined by an unknown parameter $\theta \in \Theta$ and a possibly unknown parameterization $g(\cdot, \cdot)$. A demonstration, $\trajxu \doteq (\trajx, \traju) \in \trajset^{\state\control}$, is a state-control trajectory which approximately solves Problem \ref{eq:fwdprob}, i.e. it satisfies all constraints and its cost is at most a factor of $\delta$ above the cost of a globally optimal solution $\trajxu^*$, i.e. $c(\trajx, \traju) \le (1+\delta)c(\trajx^*, \traju^*)$. For convenience, we summarize our frequently used notation in Appendix \ref{sec:app_notation}. In this paper, our goal is to recover the safe set $\safeset(\theta)$ and its complement, the unsafe set $\unsafeset(\theta) \doteq \safeset(\theta)^c$, given $\numsafe$ demonstrations $\{\traj_{s_j}^* \}_{j=1}^{\numsafe}$, $\numunsafe$ inferred unsafe trajectories $\{\traj_{\neg s_k} \}_{k=1}^{\numunsafe}$, the cost function $c_\task(\cdot)$, task-dependent constraints $\safeset_\task$, and a simulator generating dynamically-feasible trajectories satisfying control constraints.
\vspace{-8pt}
\section{Method}
\vspace{-8pt}

In this section, we describe our method (a full algorithm block is presented in Appendix \ref{sec:app_algorithm}). In Section \ref{sec:method_sampling}, we describe how to sample unsafe trajectories. In Sections \ref{sec:method_recover} and \ref{sec:method_ip}, we present mixed integer programs which recover a consistent constraint for a fixed parameterization and extract volumes of guaranteed safe/unsafe states. In Section \ref{sec:incremental}, we present how our method can be extended to the case of unknown parameterizations.
\subsection{Sampling lower-cost trajectories}\label{sec:method_sampling}

\vspace{-8pt}
\begin{wrapfigure}{r}{0.45\linewidth}
\vspace{-24pt}
\includegraphics[width=\linewidth]{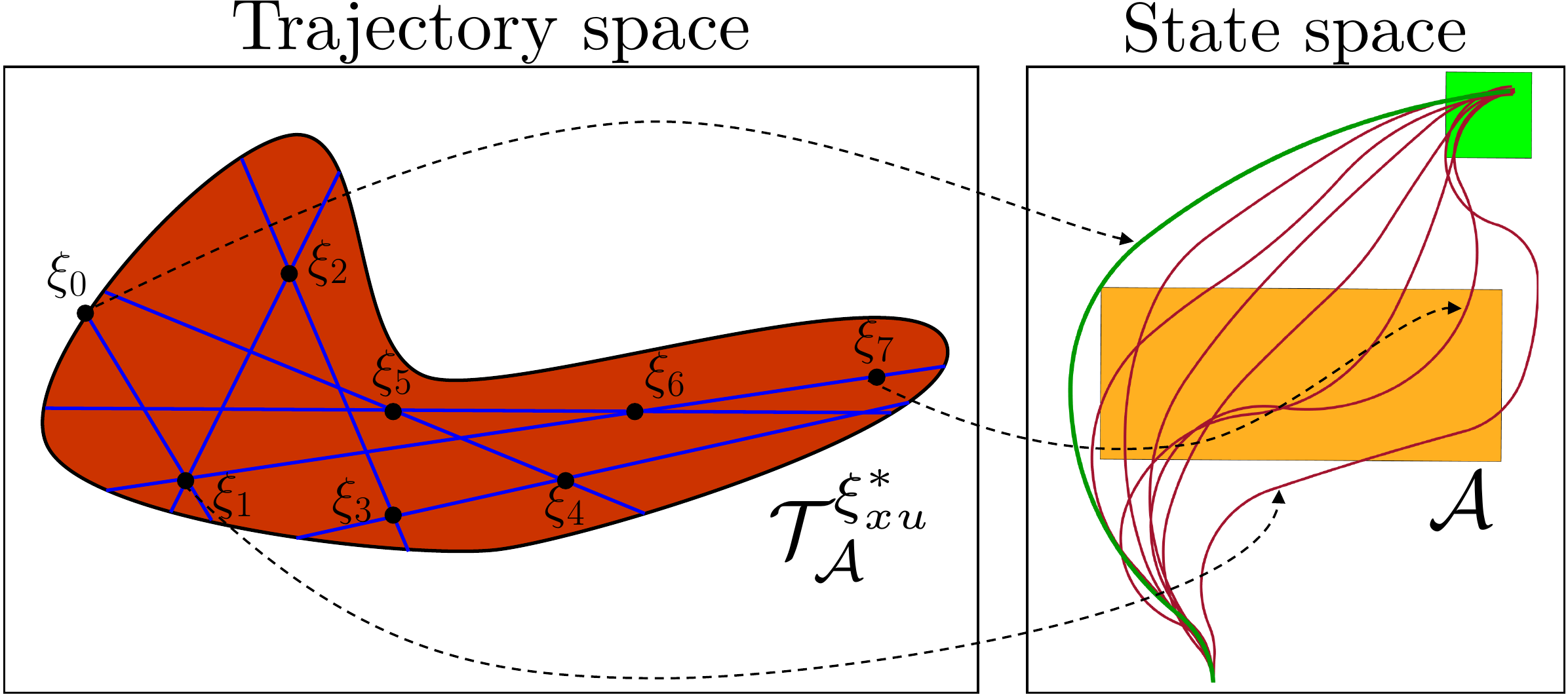}
\centering
\caption{\small Hit-and-run. \textbf{Left}: Blue lines: sampled random directions; black dots: samples. \textbf{Right}: Each point in $\mathcal{T}_{\unsafeset}^{\trajxu^*}$ corresponds to an unsafe trajectory in the constraint space $\constraintspace$ (here, $\constraintspace = \statespace$). } \label{fig:hnr}
\vspace{-25pt}
\end{wrapfigure}
In this section, we describe the general sampling framework presented in \cite{extended_version} while also relaxing the assumption of known closed-form dynamics made in \cite{extended_version}. We define the set of unsafe state-control trajectories induced by an optimal, safe demonstration $\traj_{xu}^*$, $\unsafetrajset^{\trajxu^*}$, as the set of state-control trajectories of lower cost that obey the known constraints, $\unsafetrajset^{\trajxu^*} \doteq \{ \trajxu \in \trajset_{\bar\safeset} \cap \trajset_{\safeset_\task} \mid c(\trajx, \traju) < c(\trajx^*, \traju^*)\}$. We sample from $\trajset_{\unsafeset}^{\trajxu^*}$ to obtain lower-cost trajectories obeying the known constraints using hit-and-run sampling \cite{hit_and_run}, a method guaranteeing convergence to a uniform distribution of samples over $\trajset_{\unsafeset}^{\trajxu^*}$ in the limit; an illustration is shown in Fig. \ref{fig:hnr}. Hit-and-run starts from an initial point within the set, chooses a direction uniformly at random, moves a random amount in that direction such that the new point remains within the set, and repeats \cite{extended_version}. We sample from $\trajset_{\unsafeset}^{\trajxu^*}$ indirectly by sampling control sequences and rolling them out through the dynamics to generate dynamically-feasible trajectories. We emphasize that $f(\state, \control, t)$ does not need to be known in closed form. Given a control sequence sampled by hit-and-run, a simulator can instead be used to output the resulting dynamically-feasible trajectory, which can then be checked for membership in $\trajset_{\unsafeset}^{\trajxu^*}$ exactly as if the dynamics were known in closed form. Also, $\delta$-suboptimality of the demonstration $\trajxu^\textrm{dem}$ can be handled in this framework by sampling instead from $\{ \trajxu \in \trajset_{\bar\safeset} \cap \trajset_{\safeset_\task} \mid c(\trajx, \traju) < c(\trajx^\textrm{dem}, \traju^\textrm{dem})/(1+\delta)\}$. Optimal substructure in the cost function can be exploited to sample unsafe sub-trajectories over shorter time windows on the demonstrations; shorter unsafe trajectories provide less ambiguous information regarding $\unsafeset$ and can better reduce the ill-posedness of the constraint recovery problem \cite{extended_version}.
\vspace{-12pt}
\subsection{Recovering the constraint}\label{sec:method_recover}
\vspace{-8pt}
Recall that the unsafe set can be described by some parameterization $\unsafeset(\theta) \doteq \{ \cstate \in \constraintset\ |\ g(\cstate, \theta) \le 0 \}$, where we assume for now that $g(\cdot, \cdot)$ is known, and $\theta$ are parameters to be learned. Intuitively, $g(\cstate, \theta)$ tells us if constraint state $\cstate$ (which is any element of constraint space $\constraintspace$) is safe according to parameter $\theta$. Then, a feasibility problem can be written to find a $\theta$ consistent with the data:

\begin{problem}[Parametric constraint recovery problem]\label{prob:parametric_feasibility_program}
\begin{subequations}\label{eq:parametric_feasibility_program}
\vspace{-3pt}
\begin{align}
		\textnormal{find}\quad  & \theta\notag \\
		\textnormal{s.t.}\quad & g(\cstate_i, \theta) > 0, \quad \forall \cstate_i \in \phi(\traj_{s_j}^*),\quad  \forall j = 1, \ldots, \numsafe\label{subeq:safe} \\
		 & \exists \cstate_i \in \phi(\traj_{{\neg s}_k}), \quad g(\cstate_i, \theta) \le 0, \quad  \forall k = 1, \ldots, \numunsafe\label{subeq:unsafe}
	\end{align}
\end{subequations}
\end{problem}
\vspace{-5pt}
Constraint \eqref{subeq:safe} enforces that each safe constraint state lies outside $\unsafeset(\theta)$ and constraint \eqref{subeq:unsafe} enforces that at least one constraint state on each unsafe trajectory lies inside $\unsafeset(\theta)$. Denote $\feas$ as the feasible set of Problem \ref{prob:parametric_feasibility_program}. Further denote $\guarunsafe$ and $\guarsafe$ as the set of constraint states which are learned guaranteed unsafe and safe, respectively; that is, a constraint state $\cstate \in \guarunsafe$ or $\cstate \in \guarsafe$ if $\cstate$ is classified unsafe or safe for all $\theta \in \feas$:
\begin{minipage}{.5\linewidth}
	\begin{equation}\label{eq:guarunsafe}
		\guarunsafe \doteq \bigcap_{\theta \in \feas} \{ \cstate\ |\ g(\cstate, \theta) \le 0 \}
	\end{equation}
\end{minipage}\begin{minipage}{.5\linewidth}
	\begin{equation}\label{eq:guarsafe}
		\guarsafe \doteq \bigcap_{\theta \in \feas} \{ \cstate\ |\ g(\cstate, \theta) > 0 \}
	\end{equation}
\end{minipage}

\vspace{-5pt}
In Problem \ref{prob:parametric_feasibility_program}, it is possible to learn that a constraint state is guaranteed safe/unsafe even if it does not lie directly on a demonstration/unsafe trajectory. This is due to the parameterization: for the given set of safe and unsafe trajectories, there may be no feasible $\theta \in \feas$ where $\cstate$ is classified unsafe/safe. It is precisely this extrapolation which will enable us to learn constraints in high-dimensional spaces. We now identify classes of parameterizations for which Problem \ref{prob:parametric_feasibility_program} can be efficiently solved:

\begin{wrapfigure}{r}{0.62\linewidth}
\vspace{-17pt}
\begin{problem}[Polytopic constraint recovery problem]\label{prob:parametric_polytope_program}
\begin{subequations}\label{eq:parametric_polytope_program}
	\small\begin{align}
		\hspace{-5pt}\textnormal{find}\quad  & \theta, \{b_s^i\}_{i=1}^{N_s}, \{b_{\neg s}^i\}_{i=1}^{N_{\neg s}}\notag \\
		\textnormal{s.t.}\quad & H(\theta)\cstate_i > h(\theta) - M(1 - b_s^i),\quad b_{s_j}^i \in \{0, 1\}^{N_h},\notag\\[-5pt]
		& \ \ \sum_{i=1}^{N_h} b_{s_j}^i \ge 1, \forall \cstate_i \in \phi(\traj_{s_j}^*), i = 1, ..., T_j, j = 1, ..., \numsafe\hspace{-5pt}\label{subeq:cstr1} \\[-5pt]
		 & H(\theta) \cstate_i \le h(\theta) + M(1 - b_{\neg s_k}^i)\mathbf{1}_{N_h}, \quad b_{\neg s_k}^i \in \{0, 1\}, \notag\\[-2pt]
		 & \ \ \sum_{i=1}^{T_k} b_{\neg s_k}^i \ge 1, \quad \forall \cstate_i \in \phi(\traj_{\neg s_k}),\quad \forall k = 1, ..., \numunsafe\hspace{-5pt}\label{subeq:cstr2}
	\end{align}
\end{subequations}
\end{problem}
\vspace{-30pt}
\end{wrapfigure}
\vspace{-5pt}
\textbf{Linear case}: $g(\cstate, \theta)$ is defined by a Boolean conjunction of linear inequalities, i.e. $\unsafeset(\theta)$ can be defined as the union and intersection of half-spaces. For this case, mixed-integer programming can be employed. If $g(\cstate, \theta) \le 0$ is a single polytope, i.e. $g(\cstate, \theta) \le 0 \Leftrightarrow H(\theta)k \le h(\theta)$, where $H(\theta)$ and $h(\theta)$ are affine in $\theta$, we can solve Problem \ref{prob:parametric_polytope_program}, a mixed integer feasibility problem, to find a feasible $\theta$. In Problem \ref{prob:parametric_polytope_program}, $M$ is a large positive number and $\mathbf{1}_{N_h}$ is a vector of ones of length $N_h$, where $N_h$ is the number of rows in $H(\theta)$. Constraints (\ref{subeq:cstr1}) and (\ref{subeq:cstr2}) use the big-M formulation \cite{bigM} to enforce that each safe constraint state lies outside $\unsafeset(\theta)$ and that at least one constraint state on each unsafe trajectory lies inside $\unsafeset(\theta)$. Similar problems can be solved when the safe/unsafe set can be described by unions of polytopes. As an alternative to integer programming, satisfiability modulo theories (SMT) solvers \cite{smt} can also be used to solve Problem \ref{prob:parametric_feasibility_program} if $g(\cstate, \theta)$ is defined by a Boolean conjunction of linear inequalities.
\vspace{-3pt}

\textbf{Convex case}: $g(\cstate, \theta)$ is defined by a Boolean conjunction of convex inequalities, i.e. $\unsafeset(\theta)$ can be described as the union and intersection of convex sets. For this case, satisfiability modulo convex optimization (SMC) \cite{smc} can be employed to find a feasible $\theta$.
\vspace{-5pt}

We close this subsection with some remarks on implementation and extensions to Problems \ref{prob:parametric_feasibility_program} and \ref{prob:parametric_polytope_program}.

\vspace{-8pt}
\begin{itemize}[leftmargin=*]
	\item For suboptimal demonstrations / imperfect lower-cost trajectory sampling, Problem \ref{prob:parametric_polytope_program} can become infeasible. To address this, slack variables can be introduced: replace constraint $\sum_{i=1}^{T_k} b_{\neg s_k}^i \ge 1$ with $\sum_{i=1}^{T_k} b_{\neg s_k}^i \ge v_k, v_k \in \{0, 1\}$ and change the feasibility problem to minimization of $\sum_{k=1}^{\numunsafe} (1 - v_k)$; this finds a $\theta$ that is consistent with as many unsafe trajectories as possible.
	\item In addition to recovering sets of guaranteed learned unsafe and safe constraint states, a probability distribution over possibly unsafe constraint states can be estimated by sampling unsafe sets $\unsafeset(\theta)$ from the feasible set of Problem \ref{prob:parametric_feasibility_program} using hit-and-run sampling, starting from a feasible $\theta$.
\end{itemize}
\vspace{-10pt}
\subsection{Extracting guaranteed safe and unsafe states}\label{sec:method_ip}
\vspace{-4pt}
One can check if a constraint state $\cstate \in \guarsafe$ or $\cstate \in \guarunsafe$ by adding a constraint $g(\cstate, \theta) \le 0$ or $g(\cstate, \theta) > 0$ to Problem \ref{prob:parametric_feasibility_program} and checking feasibility of the resulting program; if the program is infeasible, $\cstate \in \guarsafe$ or $\cstate \in \guarunsafe$. In other words, solving this modified integer program can be seen as querying an oracle about the safety of a constraint state $\cstate$. The oracle can then return that $\cstate$ is guaranteed safe (program infeasible after forcing $\cstate$ to be unsafe), guaranteed unsafe (program infeasible after forcing $\cstate$ to be safe), or unsure (program remains feasible despite forcing $\cstate$ to be safe or unsafe).

Unlike the gridded formulation in \cite{extended_version}, Problem \ref{prob:parametric_feasibility_program} works in the continuous constraint space. Thus, it is not possible to exhaustively check if each $\cstate \in\guarunsafe$ or $\cstate \in \guarsafe$. To address this, the neighborhood of some constraint state $\cstate_\textrm{query}$ can be checked for membership in $\guarunsafe$ by solving the following problem:

\begin{problem}[Volume extraction]\label{prob:parametric_volume_program}
\vspace{-2pt}
\begin{equation*}\label{eq:parametric_volume_program}
	\hspace{-0pt}\begin{array}{>{\displaystyle}c >{\displaystyle}l >{\displaystyle}l}
		\underset{\theta, \varepsilon}{\textnormal{min}} & \varepsilon \\
		\textnormal{s.t.} & g(\cstate_i, \theta) > 0, \quad \forall \cstate_i \in \phi(\traj_{s_j}^*),\quad  \forall j = 1, \ldots, \numsafe \\
		 & \exists \cstate_i \in \phi(\traj_{{\neg s}_k}), \quad g(\cstate_i, \theta) \le 0, \quad  \forall k = 1, \ldots, \numunsafe \\
		 & \exists \cstate_\textrm{near} \in \{\cstate_\textrm{near} \mid \Vert \cstate_\textrm{near} - \cstate_\textrm{query} \Vert_\infty \le \varepsilon \}, \quad g(\cstate_\textrm{near}, \theta) > 0
	\end{array}
\end{equation*}
\end{problem}
\vspace{-10pt}

In words, Problem \ref{prob:parametric_volume_program} finds the smallest $\varepsilon$-hypercube centered at $\cstate_\textrm{query}$ containing a $\cstate \notin \guarunsafe$; thus, any hypercube of size $\hat\varepsilon < \varepsilon$ is contained within $\guarunsafe$: $\{\cstate \mid \Vert \cstate - \cstate_\textrm{query} \Vert_\infty \le \hat\epsilon \} \subseteq \guarunsafe$. We can write a similar problem to check the neighborhood of $\cstate_\textrm{query}$ for membership in $\guarsafe$. For some common parameterizations (axis-aligned hyper-rectangles, convex sets), there are even more efficient methods for recovering subsets of $\guarsafe$ and $\guarunsafe$, which are described in Appendix \ref{sec:app_extraction}. Volumes of safe/unsafe space can thus be produced by repeatedly solving Problem \ref{prob:parametric_volume_program} for different $\cstate_\textrm{query}$, and these volumes can be passed to a planner to generate new trajectories that are guaranteed safe.
\vspace{-5pt}

\vspace{-1pt}
\subsection{Unknown parameterizations}\label{sec:incremental}
\vspace{-1pt}
For many realistic applications, we do not have access to a known parameterization which can represent the unsafe set. Despite this, complex unsafe/safe sets can often be approximated as the union of many simple unsafe/safe sets. Along this line of thought, we present a method for incrementally growing a parameterization based on the complexity of the demonstrations and unsafe trajectories.

Suppose that the true parameterization $g(\cstate, \theta)$ of the unsafe set $\unsafeset(\theta) = \{\cstate \mid g(\cstate, \theta) \le 0\}$ is unknown but can be exactly or approximately expressed as the union of $\numbox$ simple sets $\unsafeset(\theta) \approxeq \bigcup_{i=1}^{\numbox}\{\cstate\mid g_s(\cstate,\theta_i) \le 0\}\doteq \bigcup_{i=1}^{\numbox} \unsafeset(\theta_i)$, where each simple set $\unsafeset(\theta_i)$ has a known parameterization $g_s(\cdot, \cdot)$ and $\numbox$, the minimum number of simple sets needed to reconstruct $\unsafeset$, is unknown. 

A lower bound on $\numbox$, $\underline N$, can be estimated by incrementally adding simple sets until Problem \ref{prob:parametric_feasibility_program} becomes feasible. However, for $\underline N < \numbox$, the extracted $\guarsafe$ and $\guarunsafe$ are not guaranteed to be conservative estimates of $\safeset$ and $\unsafeset$ (Theorem \ref{thm:lowerboundconservative}), and $\guarsafe$ and $\guarunsafe$ are only guaranteed to be conservative if $\hat N \ge \numbox$, where $\hat N$ is the chosen number of simple sets (see Theorem \ref{thm:upperboundconservative}). Unfortunately, inferring a guaranteed overestimation of $\numbox$ only from data is not possible, as there can always be subsets of the constraint which are not activated by the given demonstrations. Two facts mitigate this:
\vspace{-5pt}

\begin{itemize}[leftmargin=*]
	\item If an upper bound on the number of simple sets needed to describe $\unsafeset(\theta)$, $\bar N_\textrm{loose} \ge N^*$, is known (where this bound can be trivially loose), $\guarsafe \subseteq \safeset$ and $\guarunsafe \subseteq \unsafeset$ by using $\bar N_\textrm{loose}$ simple sets in solving Problem \ref{prob:parametric_feasibility_program}. Hence, by using $\bar N_\textrm{loose}$, $\guarsafe$ and $\guarunsafe$ can be made guaranteed conservative (see Theorem \ref{thm:upperboundconservative}), at the cost of the resulting $\guarsafe$ and $\guarunsafe$ being potentially small.
	\item As the demonstrations begin to cover the space, $\underline N \rightarrow \numbox$. Hence, by using $\underline N$ simple sets, $\guarsafe$ and $\guarunsafe$ are asymptotically conservative.
\end{itemize}
\vspace{-5pt}
In our experiments, we choose our simple sets as axis-aligned hyper-rectangles in $\constraintspace$, which is motivated by: 1) any open set in $\constraintspace$ can be approximated as a countable/finite union of open axis-aligned hyper-rectangles \cite{tao}; 2) unions of hyper-rectangles are easily representable in Problem \ref{prob:parametric_polytope_program}.

\vspace{-5pt}
\section{Theoretical Analysis}\label{sec:theory}
\vspace{-3pt}
In this section, we present theoretical analysis on our parametric constraint learning algorithm. In particular, we analyze the conditions under which our algorithm is guaranteed to learn a conservative estimate of the safe and unsafe sets. For space, the proofs and additional results on conservativeness (Section \ref{sec:app_conservativeness}) and the learnability of a constraint (Section \ref{sec:app_learnability}) are presented in the appendix. We develop the theory for $\constraintspace = \statespace$ for legibility, but the results can be easily extended to general $\constraintspace$.

\begin{theorem}[Conservativeness: Known parameterization]\label{thm:knownconservative}
	Suppose the parameterization $g(\state, \theta)$ is known exactly. Then, for a discrete-time system, extracting $\guarunsafe$ and $\guarsafe$ (as defined in \eqref{eq:guarunsafe} and \eqref{eq:guarsafe}, respectively) from the feasible set of Problem \ref{prob:parametric_feasibility_program} returns $\guarunsafe \subseteq \unsafeset$ and $\guarsafe \subseteq \safeset$. Further, if the known parameterization is $H(\theta) \state_i \le h(\theta)$ and $M$ in Problem \ref{prob:parametric_polytope_program} is chosen to be greater than $$\max\Big(\max_{x_i \in \traj_s} \max_\theta \max_j(H(\theta) x_i - h(\theta))_j, \max_{x_i \in \traj_{\neg s}} \max_\theta \max_j (H(\theta) x_i - h(\theta))_j\Big),$$ then extracting $\guarunsafe$ and $\guarsafe$ from the feasible set of Problem \ref{prob:parametric_polytope_program} recovers $\guarunsafe \subseteq \unsafeset$ and $\guarsafe \subseteq \safeset$.
\end{theorem}
\vspace{-4pt}
We also present conservativeness results for continuous-time dynamics in Corollary \ref{thm:app_c2d}.

Now, let's consider the case where the true parameterization is not known and we use the incremental method described in Section \ref{sec:incremental}, where $g_s(\state, \theta)$ is the simple parameterization. We consider the over-parameterized case (Theorem \ref{thm:upperboundconservative}) and the under-parameterized case (Theorem \ref{thm:lowerboundconservative}). We analyze the case where the true, under-, and over-parameterization are defined respectively as:

\vspace{-11pt}
\begin{minipage}{.4\linewidth}
	\small\begin{equation}\label{eq:trueparam}
	g(\state, \theta) \le 0 \Leftrightarrow \\ \bigvee_{i=1}^{\numbox} \big(g_s(\state, \theta_i) \le 0\big)
\end{equation}
\end{minipage}
\begin{minipage}{.6\linewidth}
	\small\begin{equation}\label{eq:underparam}
	g(\state, \theta) \le 0 \Leftrightarrow \bigvee_{i=1}^{\underline N} \big(g_s(\state, \theta_i) \le 0\big), \quad \underline N < \numbox
\end{equation}
\end{minipage}
\vspace{-5pt}
\centerline{\begin{minipage}{.5\linewidth}
\vspace{-18pt}
	\small\begin{equation}\label{eq:overparam}
	g(\state, \theta) \le 0 \Leftrightarrow \bigvee_{i=1}^{\bar N} \big(g_s(\state, \theta_i) \le 0\big), \quad \bar N > \numbox.
\end{equation}
\end{minipage}}

\begin{theorem}[Conservativeness: Over-parameterization]\label{thm:upperboundconservative}
	Suppose the true parameterization and over-parameterization are defined as in \eqref{eq:trueparam} and \eqref{eq:overparam}. Then, $\guarunsafe\subseteq \unsafeset$ and $\guarsafe\subseteq \safeset$.
\end{theorem}

\begin{theorem}[Conservativeness: Under-parameterization]\label{thm:lowerboundconservative}
Suppose the true parameterization and under-parameterization are defined as in \eqref{eq:trueparam} and \eqref{eq:underparam}. Furthermore, assume that we incrementally grow the parameterization as described in Section \ref{sec:incremental}. Then, the following are true:
\vspace{-5pt}
\begin{enumerate}[leftmargin=*]
	\item $\guarunsafe$ and $\guarsafe$ are not guaranteed to be contained in $\unsafeset$ (unsafe set) and $\safeset$ (safe set), respectively.\vspace{-8pt}
	\item Each recovered simple unsafe set $\unsafeset(\theta_i)$, $i=1,\ldots, \underline N$, for any $\theta_1, \ldots, \theta_{\underline N} \in \feas$, touches the true unsafe set (there are no spurious simple unsafe sets): for $i = 1, \ldots, \underline N$, for $\theta_1, \ldots, \theta_{\underline N} \in \feas$, $\unsafeset(\theta_i) \cap \unsafeset \ne \emptyset$ ($\underline N$ is as defined in Section \ref{sec:incremental}).
\end{enumerate}
\end{theorem}
\vspace{-10pt}
\section{Results}
\vspace{-3pt}
We evaluate our method, showing that our method can be applied to constraints with unknown parameterizations (Section \ref{sec:res_unknown}), high-dimensional constraints defined for high-dimensional systems (Section \ref{sec:res_highdim}), and settings where the dynamics are not known in closed form (Section \ref{sec:res_modelfree}). We also compare our performance with a neural network (NN) baseline\footnote{In all experiments, 1) the NN is trained with the safe/unsafe trajectories and predicts at test time if a queried constraint state is safe/unsafe; 2) error bars are generated by initializing the NN with 10 different random seeds and evaluating accuracy after training. The architectures/training details are presented in Appendix \ref{sec:app_experimental}.}. We further compare with the grid-based method \cite{extended_version} on the 2D examples. For space, experimental details are provided in Appendix \ref{sec:app_experimental}.
\vspace{-10pt}

\begin{figure}
	\centering
	\includegraphics[width=\linewidth]{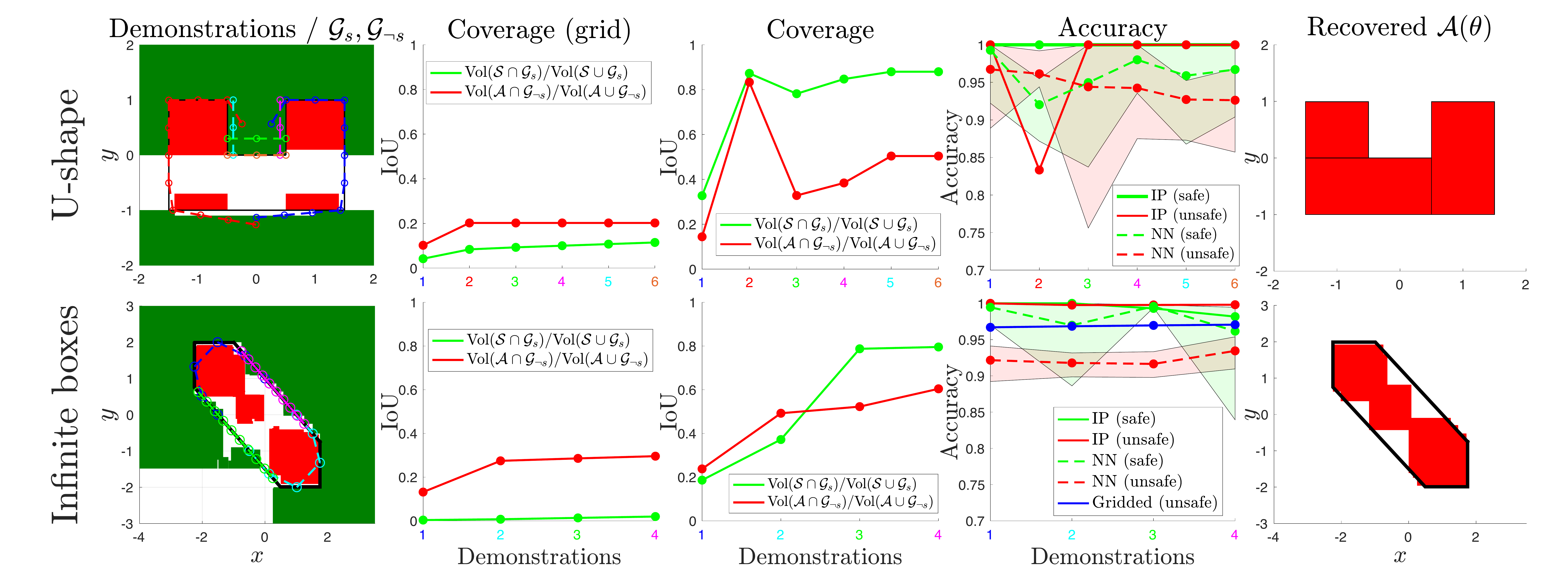}
     \vspace{-12pt}
     \caption{\small Unknown parameterization. \textbf{Col. 1}: Red: $\guarunsafe$; Green: $\guarsafe$. Demonstrations are overlaid. \textbf{Col. 2}: Coverage of $\unsafeset$ and $\safeset$ with \cite{extended_version}. In this (and all later examples), the demonstrations are color-coded with $x$-axis. \textbf{Col. 3}: Coverage of $\unsafeset$ and $\safeset$ with our method. \textbf{Col. 4}: Classification accuracy (dotted: average NN accuracy, shaded: range of NN accuracies over 10 random seeds). \textbf{Col. 5}: Recovered constraint with multi-polytope variant of Problem \ref{prob:parametric_polytope_program}.}
     \vspace{-13pt}
     \label{fig:res_unknown}
\end{figure}

\subsection{Unknown parameterization}\label{sec:res_unknown}
\vspace{-5pt}

\textbf{U-shape}: We first present a kinematic 2D example where a U-shape $\unsafeset$ is to be learned, but the number of simple unsafe sets needed to represent $\unsafeset$ (three) is unknown. In Row 1, Column 1 of Fig. \ref{fig:res_unknown}, we outline $\unsafeset$ in black and overlay $\guarunsafe$, $\guarsafe$, and the six provided demonstrations, synthetically generated via trajectory optimization. We note that due to the chosen control constraints and U-shape, there are parts of $\unsafeset$ (a subset of the white region in Fig. \ref{fig:res_unknown}, Row 1, Column 1) which cannot be implied unsafe by sampled unsafe trajectories and the parameterization (see Theorem \ref{thm:learnability_dt_parametric}). As a result, $\guarunsafe$ may not fully cover $\unsafeset$, even with more demonstrations (Fig. \ref{fig:res_unknown}, Row 1, Column 3). Note that the decrease in coverage\footnote{Coverage is measured as the intersection over union (IoU) of the relevant sets (see legends for exact formula).} at the third demonstration is due to a increase from a two-box parameterization to a three-box parameterization. Likewise, the accuracy\footnote{In all experiments, computed accuracies are: IP (safe) = $\textrm{Vol}(\guarsafe \cap \safeset)/\textrm{Vol}(\guarsafe)$, IP (unsafe) = $\textrm{Vol}(\guarunsafe \cap \unsafeset)/\textrm{Vol}(\guarunsafe)$, NN (safe) = $(\sum_{i=1}^q \mathbf{I}_{(\state_i \in \safeset) \wedge (\textrm{NN classified } \state_i \textrm{ as safe} )})/\sum_{i=1}^q \mathbf{I}_{\state_i \in \safeset}$, NN (unsafe) = $(\sum_{i=1}^q \mathbf{I}_{(\state_i \in \unsafeset) \wedge (\textrm{NN classified } \state_i \textrm{ as unsafe}})/\sum_{i=1}^q \mathbf{I}_{\state_i \in \unsafeset}$, where $\state_1, \ldots, \state_q$ are query states sampled from $\guarunsafe \cup \guarsafe$ and $\mathbf{I}_{(\cdot)}$ is the indicator function. Note that NN accuracy is computed only on $(\guarsafe \cup \guarunsafe) \subseteq \constraintspace$.} decreases at the second demonstration due to over-approximation of $\unsafeset$ with two boxes (Fig. \ref{fig:res_unknown}, Row 1, Column 4), but this over-approximation vanishes when switching to the three-box parameterization (which is exact; hence $\guarsafe$ and $\guarunsafe$ are guaranteed conservative, c.f. Theorem \ref{thm:knownconservative}). The grid-based method in \cite{extended_version} always has perfect accuracy, since it does not extrapolate beyond the observed trajectories. However, as a result of that, it also yields low coverage (Fig. \ref{fig:res_unknown}, Row 1, Column 2). The NN baseline achieves lower accuracy for the unsafe set as it misclassifies some corners of the U. Recovering a feasible $\theta$ using a multi-box variant of Problem \ref{prob:parametric_polytope_program} recovers $\unsafeset$ exactly (Fig. \ref{fig:res_unknown}, Row 1, Column 5). Finally, we note that in this (and future) examples, demonstrations were specifically chosen to be informative about the constraint. We present a version of this example in Appendix \ref{sec:app_results} with random demonstrations and show that the constraint is still learned (albeit needing more demonstrations).
\vspace{-3pt}

\textbf{Infinite boxes}: To show that our method can still learn a constraint that cannot be easily expressed using a chosen parameterization, we limit our parameterization to an unknown number of axis-aligned boxes and attempt to learn a diagonal ``I" unsafe set (see Fig. \ref{fig:res_unknown}, Row 2). This is a particularly difficult example, since an infinite number of axis-aligned boxes will be needed to recover $\unsafeset$ exactly. However, for finite data, only a finite number of boxes will be needed; in particular, for 1, 2, 3, and 4 demonstrations (which are synthetically generated assuming kinematic system constraints), 3, 5, 6, and 6 boxes are required to generate a parameterization consistent with the data (see Fig. \ref{fig:res_unknown}, Row 2, Column 1). Also overlaid in Fig. \ref{fig:res_unknown}, Row 2, Column 1 are $\guarunsafe$ and $\guarsafe$, which are approximated by solving Problem \ref{prob:parametric_volume_program} for randomly sampled $\cstate_\textrm{center}$. Compared to the gridded formulation in \cite{extended_version} (see Fig. \ref{fig:res_unknown}, Row 2, Column 3), $\guarsafe$ and $\guarunsafe$ cover $\safeset$ and $\unsafeset$ far better due to the parameterization enabling the IP to extrapolate more from the demonstrations. Furthermore, we note that while the gridded case has perfect accuracy for the safe set, it does not for the unsafe set, due to grid alignment \cite{extended_version}. Overall, the multi-box variant of Problem \ref{prob:parametric_polytope_program} recovers $\unsafeset$ well (Fig. \ref{fig:res_unknown}, Row 2, Column 5), and the remaining gap can be improved with more data. Last, we note that the NN baseline reaches comparable accuracies here (Fig. \ref{fig:res_unknown}, Row 2, Column 4), since our method suffers from a few disadvantages for this particular example. First, attempting to represent the ``I" with a finite number of boxes introduces a modeling bias that the NN does not have. Second, since the system is kinematic and the constraint is low-dimensional, many unsafe trajectories can be sampled, providing good coverage of the unsafe set. We show later that for higher dimensional constraints/systems with highly constrained dynamics, it becomes difficult to gather enough data for the NN to perform well.

\vspace{-8pt}
\subsection{High-dimensional examples}\label{sec:res_highdim}
\vspace{-3pt}
\textbf{6D pose constraint for a 7-DOF robot arm}: In this example, we learn a 6D hyper-rectangular pose constraint for the end effector of a 7-DOF Kuka iiwa arm. One such setting is when the robot is to bring a cup to a human while ensuring its contents do not spill (angle constraint) and proxemics constraints (i.e. the end effector never gets too close to the human) are satisfied (position constraint). We examine this problem for the cases of optimal and suboptimal demonstrations.

\vspace{-5pt}
\textit{Demonstration setup: } The end effector orientation (parametrized in Euler angles) and position are constrained to satisfy $(\alpha, \beta, \gamma) \in [\underline \alpha, \bar \alpha]\times [\underline \beta, \bar \beta]\times [\underline \gamma, \bar \gamma]$ and $(x, y, z) \in [\underline x, \bar x]\times [\underline y, \bar y]\times [\underline z, \bar z]$ (see Fig. \ref{fig:arm}, Column 1). For the optimal case, we synthetically generate seven demonstrations minimizing joint-space trajectory length. For the suboptimal case, five suboptimal continuous-time demonstrations approximately optimizing joint-space trajectory length are recorded in a virtual reality environment, where a human demonstrator moves the arm from desired start to goal end effector configurations using an HTC Vive (see Fig. \ref{fig:vive}). The demonstrations are time-discretized for lower-cost trajectory sampling \cite{extended_version}. In both cases, the constraint is recovered with Problem \ref{prob:parametric_polytope_program}, where $H(\theta) = [I, -I]^\top$ and $h(\theta) = \theta = [\bar x, \bar y, \bar z, \bar \alpha, \bar \beta, \bar \gamma, \underline x, \underline y, \underline z, \underline \alpha, \underline \beta, \underline \gamma]^\top$. For the suboptimal case, slack variables are added to ensure feasibility of Problem \ref{prob:parametric_polytope_program}, and for a suboptimal demonstration of cost $\hat c$, we only use trajectories of cost less than $0.9\hat c$ as unsafe trajectories.

\begin{figure}
	\vspace{-7pt}
	\centering
	\begin{subfigure}[b]{\textwidth}
         \centering
         \includegraphics[width=\linewidth]{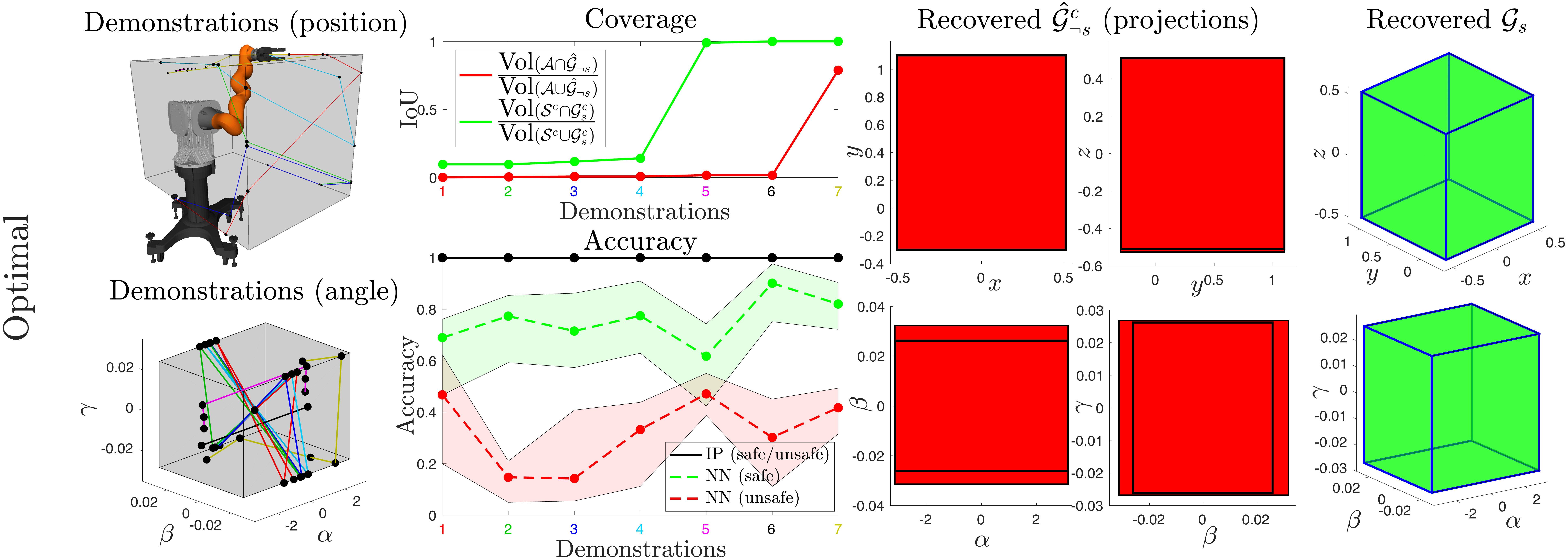}
     \end{subfigure}\vspace{-1pt}
     \begin{subfigure}[b]{\textwidth}
         \centering
         \includegraphics[width=\linewidth]{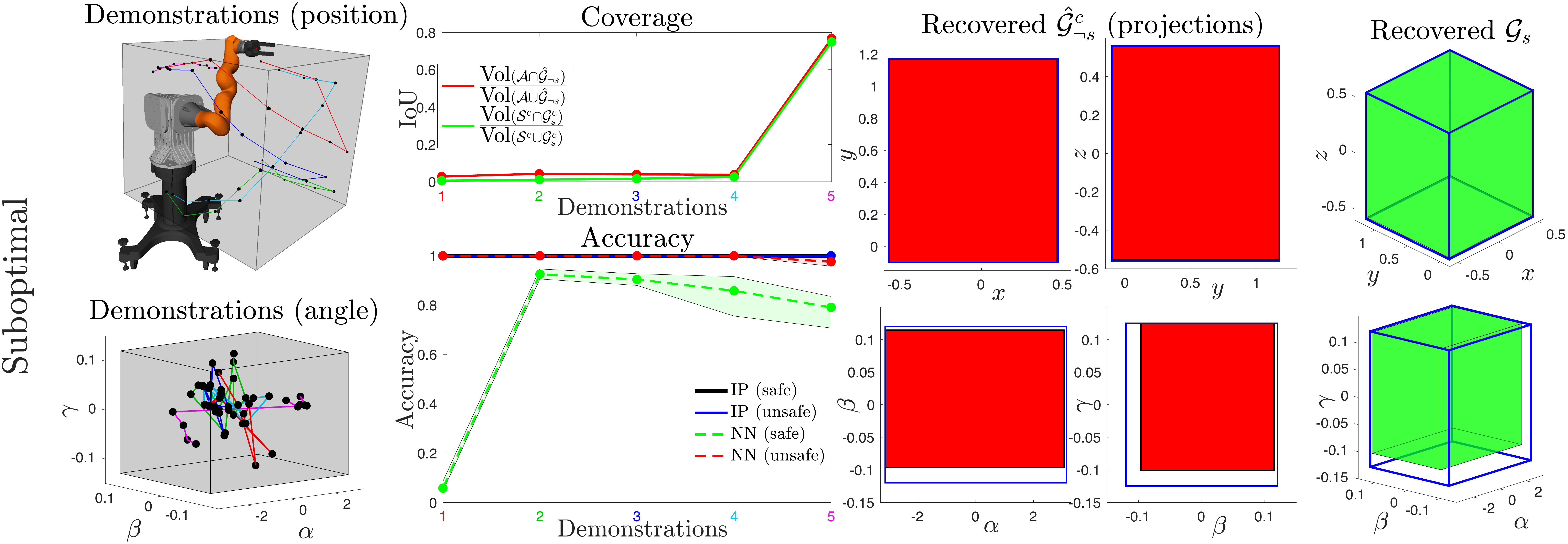}
     \end{subfigure}
     \vspace{-12pt}
     \caption{\small\textbf{Rows 1:2}: 7-DOF arm, optimal demonstrations \textbf{Col. 1}: Experimental setup. Gray boxes are projections of $\unsafeset$. Projections of demonstrations in position/angle space are overlaid. \textbf{Col. 2}: \textit{Top}: Comparing safe/unsafe set coverage as a function of demonstrations. \textit{Bottom}: Prediction accuracy. \textbf{Cols. 3-4}: projections of $\mathcal{\hat G}_{\neg s}$ using all demonstrations. For the optimal case, the red boxes over-approximate the blue boxes, as the complement of $\mathcal{\hat G}_{\neg s}$ (not $\mathcal{\hat G}_{\neg s}$ itself) is plotted. \textbf{Col. 5}: projections of $\guarsafe$ using all demonstrations. \textbf{Rows 3:4}: Same for 7-DOF arm, suboptimal demonstrations.}
     \vspace{-12pt}
     \label{fig:arm}
\end{figure}

\vspace{-5pt}
\textit{Results}: The coverage plots (Fig. \ref{fig:arm}, Rows 1 and 3, Col. 2) show that as the number of demonstrations increases, $\guarsafe$/$\guarunsafe$ approach the true safe/unsafe sets $\safeset$/$\unsafeset$ \footnote{For the unsafe sets, the IoUs are computed between $\guarunsafe^c$ and $\unsafeset^c$, as in high dimensions, the IoU changes more smoothly for the complements than the IoU between $\guarunsafe$ and $\unsafeset$, so we plot the the former for visual clarity.}. For the suboptimal case, the low IoU values for lower numbers of demonstrations is due to overapproximation of the unsafe set in the $\alpha$ component (arising from continuous-time discretization and imperfect knowledge of the suboptimality bound); the fifth demonstration, where $\alpha$ takes values near $-\pi, \pi$ greatly reduces this overapproximation. The accuracy plots (Fig. \ref{fig:arm}, Rows 2 and 4, Col. 2) present results consistent with the theory: for the optimal case, all constraint states in $\guarsafe$ and $\guarunsafe$ are truly safe and unsafe (Theorem \ref{thm:knownconservative}), and the small over-approximation for the suboptimal case is consistent with the continuous-time conservativeness (Theorem \ref{thm:app_c2d}). Note that the NN accuracy is lower and can oscillate with demonstrations, since it finds just a single constraint which is approximately consistent with the data, while our method classifies safety by consulting all possible constraints which are exactly consistent with the data, thus performing more consistently. The NN performs better on the suboptimal case than it does on the optimal case, as more unsafe trajectories are sampled due to the suboptimality, improving coverage of the unsafe set. The projections of $\mathcal{\hat G}_{\neg s}^c$ (Fig. \ref{fig:arm}, Cols. 3-4, in red), where $\mathcal{\hat G}_{\neg s} \subseteq \mathcal{G}_{\neg s}$ is obtained using the method in Appendix \ref{sec:app_extraction}, are compared to the safe set (blue outline), showing that the two match nearly exactly (though the gap for the suboptimal case is larger), and the gap can be likely reduced with more demonstrations. The projections of $\guarsafe$ (Fig. \ref{fig:arm}, Col. 5) match exactly with $\unsafeset$ for the optimal case (true safe set is outlined in blue) and match closely for the suboptimal case. Note that $\guarsafe\subseteq \safeset$, as is the case for all axis-aligned box parameterizations.

\begin{wrapfigure}{r}{0.5\linewidth}
\centering
\vspace{-13pt}
\hspace{-8pt}
\includegraphics[width=\linewidth]{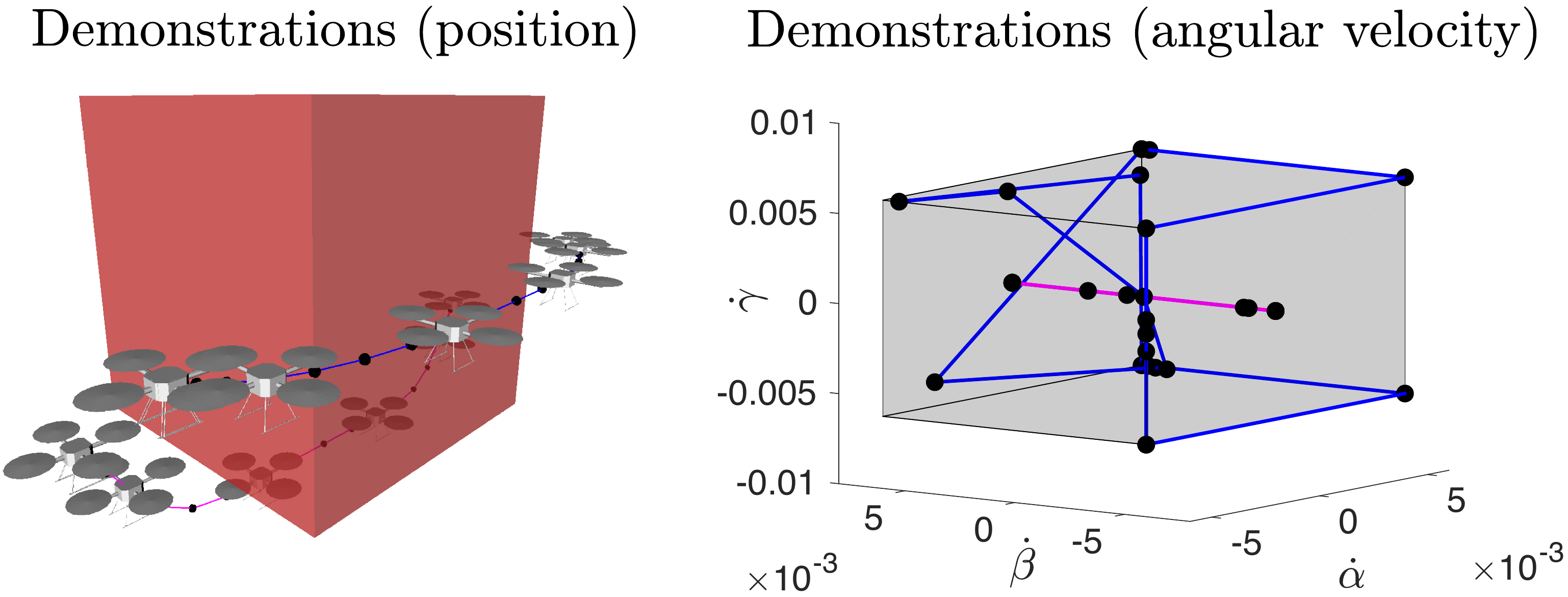}
\vspace{-6pt}
\caption{\small\textbf{Left}: Known unsafe set in $(x,y,z)$ (red); $(x,y,z)$ components of demonstrations are overlaid. \textbf{Right}: Unknown unsafe set in $(\dot\alpha,\dot\beta,\dot\gamma)$ (gray); $(\dot\alpha,\dot\beta,\dot\gamma)$ components of demonstrations are overlaid.}
\vspace{-20pt}
\label{fig:quad_demos}
\end{wrapfigure}\textbf{3D constraint for 12D quadrotor model:} We learn a 3D box angular velocity constraint for a quadrotor with discrete-time 12D dynamics (see Appendix \ref{sec:app_experimental} for details). In this scenario, the quadrotor must avoid an a priori known unsafe set in position space while also ensuring that angular velocities are below a threshold: $(\dot \alpha, \dot \beta, \dot \gamma) \in [\underline{\dot\alpha}, \bar{\dot\alpha}]\times [\underline{\dot\beta}, \bar{\dot\beta}]\times [\underline{\dot\gamma}, \bar{\dot\gamma}]$. The $(\dot\alpha, \dot\beta, \dot\gamma)$ safe set is to be inferred from two demonstrations (see Fig. \ref{fig:quad_demos}). The constraint is recovered with Problem \ref{prob:parametric_polytope_program}, where $H(\theta) = [I, -I]^\top$ and $h(\theta) = \theta = [\bar{\dot\alpha}, \bar{\dot\beta}, \bar{\dot\gamma}, \underline{\dot\alpha}, \underline{\dot\beta}, \underline{\dot\gamma}]^\top$. Fig. \ref{fig:quad_opt} shows that with more demonstrations, $\guarsafe$ approaches the true safe set $\safeset$ and $\guarunsafe$ approaches the true unsafe set $\unsafeset$, respectively. Consistent with Theorem \ref{thm:knownconservative}, our method has perfect accuracy in $\guarunsafe$ and $\guarsafe$. Here, the NN struggles more compared to the arm examples since due to the more constrained dynamics, fewer unsafe trajectories can be sampled, and a parameterization needs to be leveraged in order to say more about the unsafe set. The remaining columns of Fig. \ref{fig:quad_opt} show that we recover $\guarunsafe$ and $\guarsafe$ exactly (the true safe set is outlined in blue).

\begin{figure}
\centering
\includegraphics[width=\linewidth]{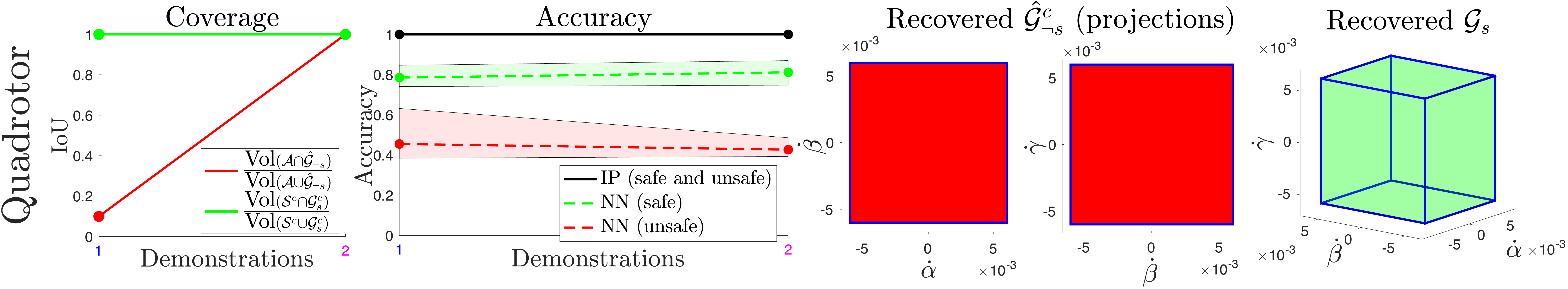}
\vspace{-17pt}
\caption{\small Constraint recovery for a 12D quadrotor. \textbf{Col. 1}: Coverage of $\unsafeset$ and $\safeset$. \textbf{Col. 2}: Classification error between $\guarsafe/\safeset$ and $\guarunsafe/\unsafeset$. \textbf{Cols. 3-4}: $\mathcal{\hat G}_{\neg s}$ using all demonstrations. \textbf{Col. 5}: $\guarsafe$ using all demonstrations.}
\label{fig:quad_opt}
\end{figure}
\vspace{-4pt}
\subsection{Planar pushing example}\label{sec:res_modelfree}
\vspace{-4pt}
\begin{figure}
\centering
\includegraphics[width=\linewidth]{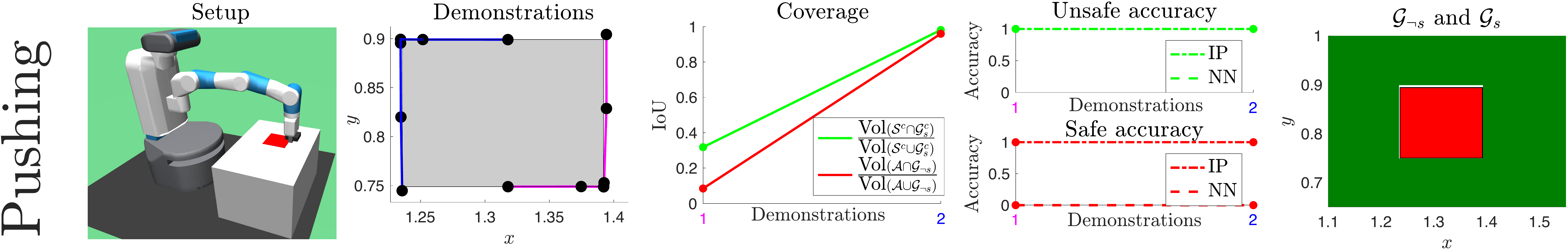}
\vspace{-17pt}
\caption{\small Constraint recovery without closed-form dynamics. \textbf{Cols. 1-2}: Setup (unsafe set in red) and demonstrations (unsafe set in gray). \textbf{Cols. 3-4}: Coverage of $\unsafeset$ and $\safeset$; classification accuracy. \textbf{Col. 5}: $\guarunsafe$ / $\guarsafe$ using all demonstrations.}
\vspace{-15pt}
\label{fig:pushing}
\end{figure}

In this section, using the $\texttt{FetchPush-v1}$ environment in OpenAI Gym \cite{openai}, we aim to learn a 2D box unsafe set on the center-of-mass (CoM) of a block pushed by the Fetch arm (see Fig. \ref{fig:pushing}) using two demonstrations. Here, the dynamics of the block CoM are not known in closed form, but rollouts can still be sampled using the simulator. Since the block CoM is highly underactuated, it is not possible to sample short sub-trajectories. Thus, without leveraging a parameterization, the constraint recovery problem is very ill-posed. Furthermore, while our method can explicitly consider the unsafeness in longer unsafe trajectories (at least one state is unsafe), the NN struggles with this example as it fails to accurately model that fact. Overall, Fig. \ref{fig:pushing} presents that $\guarunsafe$/$\guarsafe$ match up well with $\unsafeset$/$\safeset$, and our classification accuracy for safeness/unsafeness is perfect across demonstrations.

\vspace{-3pt}
\section{Discussion and Conclusion}\label{sec:conclusion}
\vspace{-3pt}
In this paper, we present a method capable of learning parametric constraints in high-dimensional spaces with and without known parameterizations. We also present a method for extracting volumes of guaranteed safe and guaranteed unsafe states, information which can be directly used in a planner to enforce safety constraints. We analyze our algorithm, showing that these recovered guaranteed safe/unsafe states are truly safe/unsafe under mild assumptions. We evaluate the method by learning a variety of constraints defined in high-dimensional spaces for systems with high-dimensional dynamics. One shortcoming of our work is scalability with the amount of data, due to the number of integer variables growing linearly with the number of safe/unsafe trajectories. As a result, learning constraints without extensive sampling of unsafe trajectories is a direction of future work.


\clearpage
\acknowledgments{This work was supported in part by a National Defense Science and Engineering Graduate (NDSEG) Fellowship, Office of Naval Research (ONR) grants N00014-18-1-2501 and N00014-17-1-2050, and National Science Foundation (NSF) grants ECCS-1553873 and IIS-1750489.}


\bibliography{example}  

\appendix
\newpage
\section{Detailed algorithm block}\label{sec:app_algorithm}

\begin{algorithm}
\SetAlgoLined
\SetKwInOut{Input}{Input}
\SetKwInOut{Output}{Output}
\Output{$\theta$ (a feasible unsafe/safe set describing the safe/unsafe trajectories),\\ $\guarsafe$ (the set of guaranteed safe constraint states),\\ $\guarunsafe$ (the set of guaranteed unsafe constraint states)}
\Input{$\traj_s = \{\traj_1^*, \ldots, \traj_{\numsafe}^*\}$, $c_\task(\cdot)$, $\textrm{known constraints}$, $\{\cstate_\textrm{query}^q\}_{q=1}^Q$}
 $\xi_{\neg s} \leftarrow \{ \}$\;
 \tcc{Sample unsafe trajectories $\traj_{\neg s}$}
 \For{i = 1:$\numsafe$}{
   $\xi_{\neg s} \leftarrow \xi_{\neg s} \cup\textsf{HitAndRun}(\traj_i^*)$\;}
  \tcc{Constraint recovery}
  $\theta \leftarrow$ Problem Y$(\traj_s, \traj_{\neg s})$\;
  \tcc{Y = \ref{prob:parametric_feasibility_program} if general parameterization}
  \tcc{Y = \ref{prob:parametric_polytope_program} if polytope parameterization}
  $\guarsafe, \guarunsafe \leftarrow \{\}, \{\}$\;
  \tcc{Guaranteed safe/unsafe recovery}
  \uIf{$\textrm{general parameterization}$}{
   \For{$q = 1,\ldots, Q$}{
   \tcc{Extract safe/unsafe volume around query point $\cstate_\textrm{query}^q$}
   $\guarsafe(\cstate_\textrm{query}^q), \guarunsafe(\cstate_\textrm{query}^q) \leftarrow$ Problem \ref{prob:parametric_volume_program}$(\cstate_\textrm{query}^q)$\;
   $\guarsafe \leftarrow \guarsafe \cup \guarsafe(\cstate_\textrm{query}^q)$\;
   $\guarunsafe \leftarrow \guarunsafe \cup \guarunsafe(\cstate_\textrm{query}^q)$\;}
   }\uElseIf{\textrm{axis-aligned hyper-rectangle parameterization}}{
   $\guarsafe, \mathcal{\hat G}_{\neg s} \leftarrow $ Procedure in Appendix \ref{sec:axis_aligned_box}\;
   }\uElseIf{\textrm{convex parameterization}}{
   $\guarsafe, \mathcal{\hat G}_{\neg s} \leftarrow $ Procedure in Appendix \ref{sec:convex}\;}
 \caption{Overall method}\label{alg:total}
\end{algorithm}

\section{Extraction of $\guarsafe$ and $\guarunsafe$}\label{sec:app_extraction}

In this section, we discuss specific ways of extracting sets of guaranteed safe/unsafe states for axis-aligned hyper-rectangles (this method is used for all numerical examples in Section \ref{sec:res_highdim} and Section \ref{sec:res_modelfree}) and for convex parameterizations.

\subsection{Axis-aligned hyper-rectangle parameterization}\label{sec:axis_aligned_box}

In this parameterization, $\constraintspace \subseteq \mathbb{R}^n$, $\theta = [\underline \cstate_1, \bar \cstate_1, \ldots, \underline \cstate_n, \bar \cstate_n]$, and $g(\cstate, \theta) \le 0 \Leftrightarrow H(\theta)\cstate \le h(\theta)$, where $H(\theta)k = [I_{n\times n}, -I_{n\times n}]^\top$ and $h(\theta) = [\bar k_1, \ldots, \bar k_n, \underline k_1, \ldots, \underline k_n]^\top$. Here, $\underline \cstate_i$ and $\bar \cstate_i$ are the lower and upper bounds of the hyper-rectangle for coordinate $i$.
	
		As the set of axis-aligned hyper-rectangles is closed under intersection, $\guarunsafe$ is also an axis-aligned hyper-rectangle, the axis-aligned bounding box of any two constraint states $\cstate_1, \cstate_2 \in \guarunsafe$ is also contained in $\guarunsafe$. This also implies that $\guarunsafe$ can be fully described by finding the top and bottom corners $[\underline \cstate_1, \ldots, \underline \cstate_n]^\top$ and $[\bar \cstate_1, \ldots, \bar \cstate_n]^\top$. Suppose we start with a known $\cstate \in \guarunsafe$. Then, finding $[\underline \cstate_1, \ldots, \underline \cstate_n]^\top$ amounts to performing a binary search for each of the $n$ dimensions, and the same holds for finding $[\bar \cstate_1, \ldots, \bar \cstate_n]^\top$.
		
		Recovering $\guarsafe$ is not as straightforward, as the complement of axis-aligned boxes is not closed under intersection. While we can still solve Problem \ref{prob:parametric_volume_program} to recover $\guarsafe$, an inner approximation of $\guarsafe$ can be more efficiently obtained: starting at a constraint state $\cstate \in \guarunsafe$, $2n$ line searches can be performed to find the two points of transition to $\guarunsafe$ in each constraint coordinate. Denote as $\mathcal{\hat G}_{s}$ the complement of the axis-aligned bounding box of these $2n$ points; $\mathcal{\hat G}_{s}$ is an inner approximation of $\guarsafe$, as $\guarsafe = ( \bigcap_{\theta \in \feas} \{ x\ |\ g(x, \theta) \le 0\} )^c \supseteq \textrm{AABB}(\bigcap_{\theta \in \feas} \{ x\ |\ g(x, \theta) \le 0\})^c$, where $\textrm{AABB}(\cdot)$ denotes the axis-aligned bounding box of a set of points. For example, consider the scenario in Fig. \ref{fig:inner_approx} where there are only two feasible parameters, $\theta_1$ and $\theta_2$. Here, $\guarsafe$ is $(\unsafeset(\theta_1) \cup \unsafeset(\theta_2))^c$ and $\hat\guarsafe$ under-approximates the safe set ($\guarsafe$ is in general not representable as the complement of an axis-aligned box).
		
		\begin{figure}
			\centering
			\includegraphics[width=0.8\linewidth]{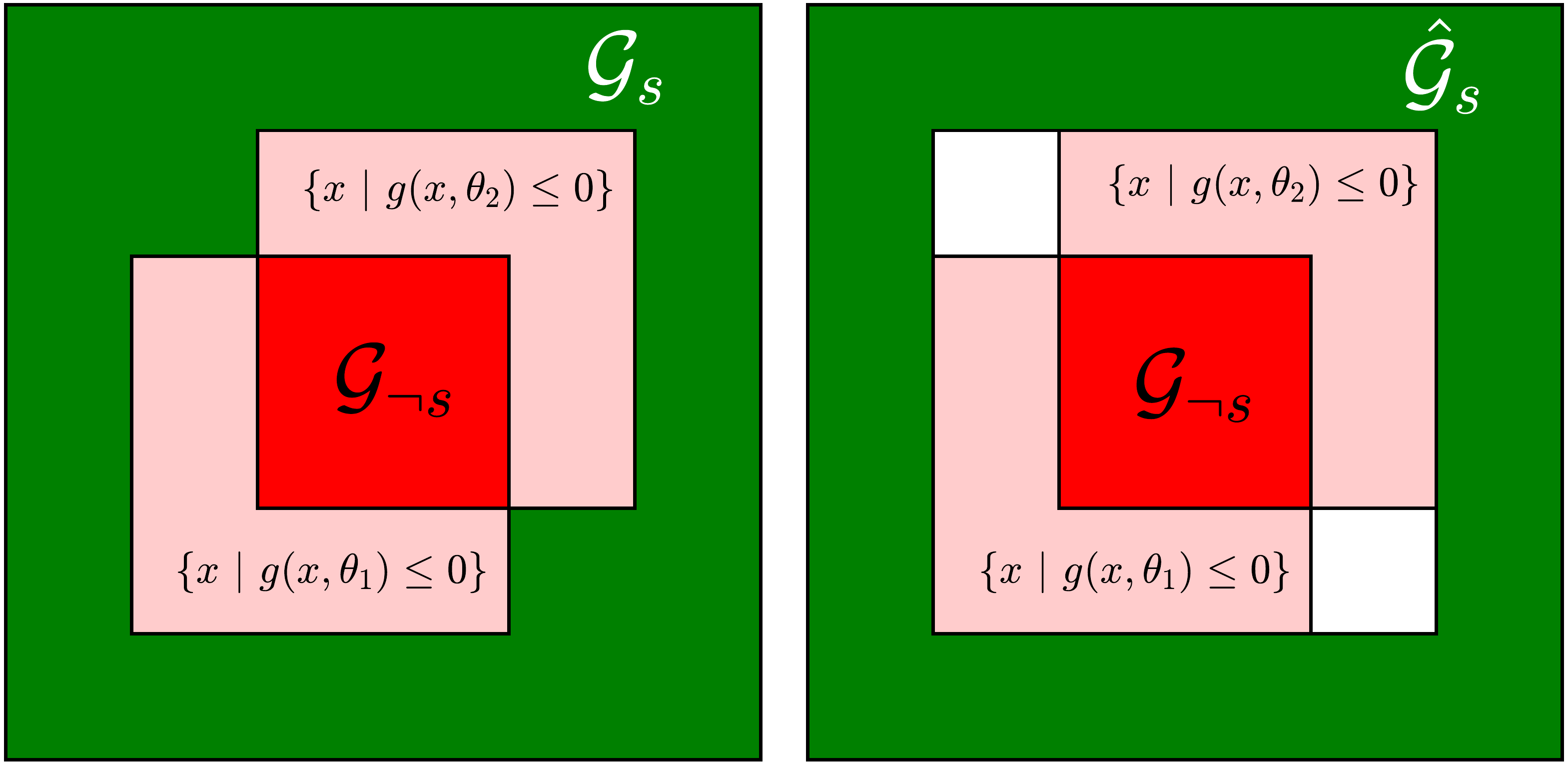}
			\caption{Comparison of the true $\guarsafe$ (left, in green) and the extracted inner approximation $\hat\guarsafe$ (right, in green).}
			\label{fig:inner_approx}
		\end{figure}

\subsection{Convex parameterization}\label{sec:convex}
In this parameterization, for fixed $\theta$, $\{ \cstate\ |\ g(\cstate, \theta) \le 0\}$ is convex.

		While apart from solving Problem \ref{prob:parametric_volume_program} it is hard to recover $\guarunsafe$ exactly, an inner approximation of $\guarunsafe$ can be extracted more efficiently by taking the convex hull of any $\cstate_1, \cstate_2, \ldots \in \guarunsafe$, as the convex hull is the minimal convex set containing $\cstate_1, \cstate_2, \ldots$.

The same approaches apply for recovering $\guarsafe$ when it is instead the safe set which is an axis-aligned hyper-rectangle or a convex set.

\section{Theoretical Analysis (Expanded)}\label{sec:app_theory}

In this section, we present theoretical analysis on our parametric constraint learning algorithm. In particular, we analyze the limits of what constraint states can be learned guaranteed unsafe/safe (Section \ref{sec:app_learnability}) as well as the conditions under which our algorithm is guaranteed to learn a conservative estimate of the safe and unsafe sets (Section \ref{sec:app_conservativeness}). For ease of reading, we repeat the theorem statements from the main body (the corresponding theorem numbers from the main body are listed in the theorem statement). We develop the theory for $\constraintspace = \statespace$ for legibility, but the results can be easily extended to general $\constraintspace$.

\subsection{Learnability}\label{sec:app_learnability}

In this section, we develop results for learnability of the unsafe set in the parametric case. We begin with the following notation:

\begin{definition}[Signed distance]
	Signed distance from point $p \in \mathbb{R}^m$ to set $\mathcal{S} \subseteq \mathbb{R}^m$, $\sd(p, \mathcal{S}) = -\inf_{y \in \partial\mathcal{S}} \Vert p - y \Vert$ if $p\in \mathcal{S}$; $\inf_{y \in \partial\mathcal{S}} \Vert p - y \Vert$ if $p\in \mathcal{S}^c$.
\end{definition}

\begin{definition}[$\Delta x$-shell]
		For a discrete time system satisfying $\Vert x_{t+1} - x_t \Vert \le \umax$ for all $t$, denote the $\umax$ shell of the unsafe set as: $\unsafeset_{\Delta x}\doteq \{ x \in \unsafeset \ |\ -\umax \le \sd(x, \unsafeset) \le 0 \}$.
\end{definition}

\begin{definition}[Implied unsafe set]
For some set $\mathcal{B} \subseteq \Theta$, denote $I(\mathcal{B}) \doteq \bigcap_{ \theta \in \mathcal{B} } \{x\ |\ g(x, \theta) \le 0 \}$ as the set of states that are implied unsafe by restricting the parameter set to $\mathcal{B}$. In words, $I(\mathcal{B})$ is the set of states for which all $\theta \in \mathcal{B}$ mark as unsafe.
\end{definition}

\begin{definition}[Feasible set $\feas$]
	Denote as $\feas$ the feasible set of Problem \ref{prob:parametric_feasibility_program} with $\numsafe$ demonstrations and $\numunsafe$ unsafe trajectories sampled using the hit-and-run method presented in Section \ref{sec:method_sampling}:
	\begin{align*}
		\feas = \{\theta\ |\ &\forall i \in \{1, \ldots, \numsafe\}, \forall x \in \traj_i^*, g(x, \theta) > 0, \\
		& \forall j \in \{1, \ldots, \numunsafe\},\exists x \in \traj_j, g(x, \theta) \le 0\}.
	\end{align*}
\end{definition}

\begin{definition}[Learnability and learnable set $\guarunsafe^*$]
	A state $\state \in \unsafeset$ is learnable if there exists \textbf{any} set of $\numsafe$ demonstrations and $\numunsafe$ unsafe trajectories sampled using the hit-and-run method presented in Section \ref{sec:method_sampling}, where $\numsafe$ and $\numunsafe$ may be infinite, such that $\state \in \mathcal{I}(\feas)$. Accordingly, we define the \textit{learnable} set of unsafe states $\guarunsafe^*$ as the union of all learnable states. Note that by this definition, a state $\state_s \in \safeset$ is always learnable, since there always exists some safe demonstration passing through $\state_s$.
\end{definition}

\begin{lemma}\label{lem:contain}
	Suppose $\mathcal{B} \subseteq \mathcal{\hat B}$, for some other set $\mathcal{\hat B}$. Then, $I(\mathcal{\hat B}) \subseteq I(\mathcal{B})$.
\end{lemma}

\begin{proof}
	By definition, 
	\begin{align*}
		I(\mathcal{\hat B}) &= \bigcap_{ \theta \in \mathcal{\hat B} } \{x\ |\ g(x, \theta) \le 0 \} \\
		&= \bigcap_{ \theta \in \big( \mathcal{B} \cup (\mathcal{\hat B} \setminus \mathcal{B}) \big) } \{x\ |\ g(x, \theta) \le 0 \} \\
		&\subseteq \bigcap_{ \theta \in \mathcal{B} } \{x\ |\ g(x, \theta) \le 0 \} \\
		&= I(\mathcal{B}).
	\end{align*}
\end{proof}

\begin{lemma}\label{lem:atleastone}
	Each unsafe trajectory $\traj_j$ with start and goal states in the safe set contains at least one state in the $\Delta x$-shell $\unsafeset_{\Delta x}$: $\forall j \in \{1, \ldots, \numunsafe\}, \exists x \in \traj_j, x \in \unsafeset_{\Delta x}$.
\end{lemma}
\begin{proof}
	For each unsafe trajectory $\traj_j$ with start and goal states in the safe set, there exists $x \in \traj_j, x\in \unsafeset$. Further, if there exists $x \in \traj_j \in (\unsafeset\setminus \unsafeset_{\Delta x})$, then there also exists $x \in \traj_j \in \unsafeset_{\Delta x}$. For contradiction, suppose there exists a time $\hat t \in \{1, \ldots, T_j\}$ for which $\traj_j(\hat t) \in (\unsafeset\setminus \unsafeset_{\Delta x})$ and $\nexists t \in \{1, \ldots, T_j\}$ for which $\traj_j(t) \in \unsafeset_{\Delta x}$. But this implies $\exists t < \hat t, \Vert \traj(t) - \traj(t+1) \Vert > \Delta x$ or $\exists t > \hat t, \Vert \traj(t) - \traj(t-1) \Vert > \Delta x$, i.e. to skip deeper than $\Delta x$ into the unsafe set without first entering the $\Delta x$ shell, the state must have changed by more than $\Delta x$ in a single time-step. Contradiction. An analogous argument holds for the continuous-time case. 
\end{proof}

The following result states that in discrete time, the learnable set of unsafe states $\guarunsafe^*$ is contained by the set of states which must be implied unsafe by setting $\unsafeset_{\Delta x}$ as unsafe. Furthermore, in continuous time, the same holds, except the $\unsafeset_{\Delta x}$ is replaced by the boundary of the unsafe set, $\partial \unsafeset$.

\begin{theorem}[Discrete time learnability for parametric constraints]\label{thm:learnability_dt_parametric}
	For trajectories generated by discrete time systems, $\guarunsafe \subseteq \guarunsafe^* \subseteq I(\feas_{\Delta x})$, where
	\begin{align*}
		\feas_{\Delta x} = \{\theta\ |\ &\forall i \in \{1, \ldots, \numsafe\},\quad \forall x \in \traj_i^*, g(x, \theta) > 0, \quad \forall x \in \unsafeset_{\Delta x}, g(x, \theta) \le 0\}.
	\end{align*}
\end{theorem}
\begin{proof}
	Recall that $\guarunsafe \doteq \bigcap_{\theta \in \feas} \{ \state\ |\ g(\state, \theta) \le 0 \}$, where as previously defined, $\feas$ is the feasible set of Problem \ref{prob:parametric_feasibility_program}.
	We can then show that $\feas_{\Delta x} \subseteq \feas$, since enforcing that $g(x, \theta) \le 0$ for all $x \in \unsafeset_{\Delta x}$ implies that there exists $x \in \traj_j$, for all $j \in \{1, \ldots, \numunsafe\}$ such that $g(x, \theta) \le 0$, via Lemma \ref{lem:atleastone}. Then, via Lemma \ref{lem:contain}, $\guarunsafe = I(\feas) \subseteq I(\feas_{\Delta x})$. As this holds for any arbitrary set of trajectories, $\guarunsafe^* \subseteq I(\feas_{\Delta x})$ as well, and $\guarunsafe \subseteq \guarunsafe^*$.
\end{proof}

\begin{corollary}[Continuous-time learnability for parametric constraints]
	For trajectories generated by continuous time systems, $\guarunsafe \subseteq \guarunsafe^* \subseteq I(\feas_{\partial \unsafeset})$, where
	\begin{align*}
		\feas_{\partial \unsafeset} = \{\theta\ |\ &\forall x \in \traj_i^*,\quad \forall i \in \{1, \ldots, \numsafe\}, g(x, \theta) > 0, \quad \forall x \in \partial\unsafeset, g(x, \theta) \le 0\}.
	\end{align*}
\end{corollary}
\begin{proof}
	Since going from discrete time to continuous time implies $\Delta x \rightarrow 0$, $\unsafeset_{\Delta x} \rightarrow \partial \unsafeset$. Then, the logic from the proof of Theorem \ref{thm:learnability_dt_parametric} can be similarly applied to show the result.
\end{proof}

\subsection{Conservativeness: Parametric}\label{sec:app_conservativeness}

We write conditions for conservative recovery of the unsafe set and safe set when solving Problems \ref{prob:parametric_feasibility_program} and \ref{prob:parametric_polytope_program} for discrete time and continuous time systems.

\begin{theorem}[Conservativeness: Known parameterization (Theorem \ref{thm:knownconservative} in the main body) ]\label{thm:app_knownconservative}
	Suppose the parameterization $g(\state, \theta)$ is known exactly. Then, for a discrete-time system, extracting $\guarunsafe$ and $\guarsafe$ (as defined in \eqref{eq:guarunsafe} and \eqref{eq:guarsafe}, respectively) from the feasible set of Problem \ref{prob:parametric_feasibility_program} returns $\guarunsafe \subseteq \unsafeset$ and $\guarsafe \subseteq \safeset$. Further, if the known parameterization is $H(\theta) \state_i \le h(\theta)$ and $M$ in Problem \ref{prob:parametric_polytope_program} is chosen to be greater than $$\max\Big(\max_{x_i \in \traj_s} \max_\theta \max_j(H(\theta) x_i - h(\theta))_j, \max_{x_i \in \traj_{\neg s}} \max_\theta \max_j (H(\theta) x_i - h(\theta))_j\Big),$$ then extracting $\guarunsafe$ and $\guarsafe$ from the feasible set of Problem \ref{prob:parametric_polytope_program} recovers $\guarunsafe \subseteq \unsafeset$ and $\guarsafe \subseteq \safeset$.
\end{theorem}

\begin{proof}
	
	We first prove that $\guarunsafe \subseteq \unsafeset$. Consider first the case of Problem \ref{prob:parametric_feasibility_program}, or equivalently the case of Problem \ref{prob:parametric_polytope_program} where $M = \infty$ (in this case, Problem \ref{prob:parametric_polytope_program} exactly enforces that at least one state in each unsafe trajectory is unsafe and all states on demonstrations are safe).
	
	Suppose for contradiction that there exists some $\state \in \guarunsafe, \state \notin \unsafeset$. By definition of $\guarunsafe$, $g(\state, \theta) \le 0$, for all $\theta \in \feas$, where $\feas$ is the feasible set of parameters $\theta$ in Problem \ref{prob:parametric_feasibility_program}. However, as $\state \notin \unsafeset$, but for all $\theta \in \feas, g(x, \theta) \le 0$ we know that $\theta_\unsafeset \notin \feas$, where $\theta_\unsafeset$ is the parameter associated with the true unsafe set $\unsafeset$. However, $\feas$ will always contain $\theta_\unsafeset$, since:
	\begin{itemize}[leftmargin=*]
		\item $\theta_\unsafeset$ satisfies $g(\state, \theta_\unsafeset) >0$ for all $\state$ in safe demonstrations, since all demonstrations are safe with respect to the true $\theta_\unsafeset$.
		\item For each trajectory $\traj_{\neg s}$ sampled using the hit-and-run procedure in Section \ref{sec:method_sampling}, there exists $\state \in \traj_{\neg s}$ such that $g(\state, \theta_\unsafeset) \le 0$.
	\end{itemize}
	We come to a contradiction, and hence for Problem \ref{prob:parametric_feasibility_program} and for Problem \ref{prob:parametric_polytope_program} where $M=\infty$, $\guarunsafe \subseteq \unsafeset$.
	
	Now, we consider the conditions on $M$ such that choosing $M \ge \textrm{const}$ or $M = \infty$ causes no changes in the solution of Problem \ref{prob:parametric_polytope_program}. $M$ must be chosen such that 1) $H(\theta) x_i - h(\theta) > -M \mathbf{1} \Leftrightarrow H(\theta) x_i - h(\theta) > -\infty \mathbf{1}$, for all safe states $x_i \in \traj_s$, and 2) $H(\theta) x_i - h(\theta) \le M\mathbf{1} \Leftrightarrow H(\theta) x_i - h(\theta) \le \infty\mathbf{1}$ for all states $\state_i$ on unsafe trajectories $\traj_{\neg s}$. Condition 1 is met if $-M < \min_{x_i \in \traj_s} \min_\theta \min_j(H(\theta) x_i - h(\theta))_j$, where $v_j$ denotes the $j$-th element of vector $v$; denote as $M_1$ an $M$ which satisfies this inequality. Condition 2 is met if $M \ge \max_{x_i \in \traj_{\neg s}} \max_\theta \max_j (H(\theta) x_i - h(\theta))_j$; denote as $M_2$ an $M$ which satisfies this inequality. Then, $M$ should be chosen to satisfy $M > \max(M_1, M_2)$.
	
	The proof that $\guarsafe \subseteq \safeset$ is analogous. If there exists $\state \in \guarsafe, \state \notin \safeset$, $g(\state, \theta) > 0$, for all $\theta \in \feas$, then $\theta_\unsafeset \notin \feas$. We follow the same reasoning from before to show that $\theta_\unsafeset \in \feas$ for $M = \infty$. Now, provided the condition on $M$ holds, we reach a contradiction.
\end{proof}

\begin{rem}
	A simple corollary from Theorem \ref{thm:app_knownconservative} is that by solving Problem \ref{prob:parametric_volume_program} repeatedly for different query centers $\state_\textrm{query}$ for a discrete-time system and unioning over the resulting volumes will also provide conservative estimates of $\guarsafe$ and $\guarunsafe$. Further, if the assumption on $M$ holds, then the volume extraction analogue of Problem \ref{prob:parametric_polytope_program} will also return conservative estimates of $\guarsafe$ and $\guarunsafe$.
\end{rem}

As discussed in \cite{extended_version}, with continuous-time system dynamics, assigning unsafeness in lower-cost trajectories difficult since there are an infinite number of states on the continuous trajectory. To ameliorate this, as in \cite{extended_version}, we time-discretize the sampled lower-cost trajectories and feed the resulting discrete-time trajectories into Problems \ref{prob:parametric_feasibility_program} and \ref{prob:parametric_polytope_program}. This can potentially cause a mild overapproximation of the unsafe set, which we quantify after introducing some notation.

\begin{definition}[Normal vectors]
	Denote the outward-pointing normal vector at a point $p\in\partial\unsafeset$ as $\hat n(p)$. Furthermore, at non-differentiable points on $\partial \unsafeset$, $\hat{n}(p)$ is replaced by the set of normal vectors for the sub-gradient of the Lipschitz function describing $\partial\unsafeset$ at that point (\cite{thickness}).
\end{definition}

\begin{definition}[$\gamma$-offset padding]\label{def:app_offset}
	Define the $\gamma$-offset padding $\partial \unsafeset_{\gamma}$ as:
		$\partial \unsafeset_{\gamma} = \{ x \in \statespace \ | \ x = y + d \hat n(y), d\in [0, \gamma], y \in \partial \unsafeset \}$.
\end{definition}

\begin{definition}[$\gamma$-padded set]\label{def:app_buffered_set}
	We define the $\gamma$-padded set of the unsafe set $\unsafeset$, $\unsafeset(\gamma)$, as the union of the $\gamma$-offset padding and $\unsafeset$: $\unsafeset(\gamma) \doteq \partial \unsafeset_\gamma \cup \unsafeset$.
\end{definition}

\begin{definition}[Maximum distance on trajectories]\label{def:maxdist}
	Denote $D_\traj([a, b]) \doteq \sup_{t_1 \in [a,b], t_2 \in [t_1,b]} \Vert \traj(t_1) - \traj(t_2) \Vert_2$, for some trajectory $\traj$. Denote $D^* \doteq \max_{i \in \{1, \ldots, \numunsafe\}} D_{\traj_i}^*([a_i, b_i])$. In words, $D_\traj([a, b])$ is the maximum distance between any two points on trajectory $\traj$ from time $a$ to time $b$, and $D^*$ takes the maximum distance over all $\numunsafe$ trajectories.
\end{definition}

\begin{lemma}[Maximum distance]\label{lem:app_maxdist}
	Consider a continuous time trajectory $\traj:[0,T]\rightarrow\statespace$. Suppose it is known that in some time interval $[a, b], a \le b, a, b \in [0, T]$, $\traj$ is unsafe; denote this sub-segment as $\traj([a, b])$. 
	Consider any $t \in [a, b]$. Then, the signed distance from $\traj(t)$ to the unsafe set, $\sd(\traj(t), \unsafeset)$, is bounded by $D_\traj([a, b]) \doteq \sup_{t_1 \in [a,b], t_2 \in [t_1,b]} \Vert \traj(t_1) - \traj(t_2) \Vert_2$.
\end{lemma}
\begin{proof}
	Since there exists $\tilde t \in [a,b]$ such that $\traj(\tilde t) \in \unsafeset$, $\sup_{t\in[a,b]} \sd(\traj(t), \unsafeset) = \sup_{t\in[a, b]} \sd(\traj(t), \traj(\tilde t)) \le \sup_{t_1 \in [a,b], t_2 \in [t_1,b]} \Vert \traj(t_1) - \traj(t_2) \Vert_2$.
\end{proof}

\begin{corollary}\label{thm:app_c2d}
	For a continuous-time system where demonstrations and sampled unsafe trajectories are time-discretized, if $M$ is chosen as in Theorem \ref{thm:app_knownconservative}, $\guarsafe \subseteq \safeset$, where $\safeset$ is the safe set, and $\guarunsafe \subseteq \unsafeset(D^*)$, where $D^*$ is as defined in Definition \ref{def:maxdist}.
\end{corollary}
\begin{proof}
	The reasoning for $\guarsafe \subseteq \safeset$ follows from the proof of $\guarunsafe \subseteq \unsafeset$ in the proof of Theorem \ref{thm:app_knownconservative}.
	
	Now we prove $\guarunsafe \subseteq \unsafeset(D^*)$. Suppose in this case, there exists a state $\state = \traj_j(t_i) \notin \unsafeset$ which is truly safe but lies on a sampled unsafe trajectory $\traj_{j}([a_j, b_j])$, and suppose that $\{t_1, \ldots, t_N\}$ is chosen such that for all $k \in \{1, \ldots, N\}\setminus \{i\}$, $\traj_{j}(t_k)$ belongs to a known safe cell. Then, we may incorrectly learn that $\traj_j(t_i)$ is unsafe, as we force at least one point in the sampled trajectory to be unsafe. Via Lemma \ref{lem:app_maxdist}, we know that $\traj_{j}(t_i)$ is at most $D_{\traj_j}([a_j, b_j])$ signed distance away from $\unsafeset$. Hence, for this trajectory, any learned guaranteed unsafe state must be contained in the $D_{\traj_j}([a_j, b_j])$-padded unsafe set. For this to hold for all unsafe trajectories sampled with the hit-and-run procedure presented in Section \ref{sec:method_sampling}, we must pad the unsafe set by $D^*$. Hence, under this assumption, the algorithm returns a conservative estimate of the $D^*$-padded unsafe set.
\end{proof}

Let's consider the case where the true parameterization is not known and we use the method described in Section \ref{sec:incremental}, where $g_s(\state, \theta)$ is the simple parameterization. We consider the under-parameterized case (Theorem \ref{thm:lowerboundconservative}) and the over-parameterized case (Theorem \ref{thm:upperboundconservative}). In particular, we analyze the case where the true parameterization, the under-parameterization, and the over-parameterization are defined respectively as:

\vspace{-5pt}
\begin{minipage}{.4\linewidth}
	\small\begin{equation}\label{eq:app_trueparam}
	g(\state, \theta) \le 0 \Leftrightarrow \\ \bigvee_{i=1}^{\numbox} \big(g_s(\state, \theta_i) \le 0\big)
\end{equation}
\end{minipage}
\begin{minipage}{.6\linewidth}
	\small\begin{equation}\label{eq:app_underparam}
	g(\state, \theta) \le 0 \Leftrightarrow \bigvee_{i=1}^{\underline N} \big(g_s(\state, \theta_i) \le 0\big), \quad \underline N < \numbox
\end{equation}
\end{minipage}

\vspace{-5pt}
\centerline{\begin{minipage}{.5\linewidth}
	\small\begin{equation}\label{eq:app_overparam}
	g(\state, \theta) \le 0 \Leftrightarrow \bigvee_{i=1}^{\bar N} \big(g_s(\state, \theta_i) \le 0\big), \quad \bar N > \numbox.
\end{equation}
\end{minipage}}

\begin{theorem}[Conservativeness: Over-parameterization (Theorem \ref{thm:upperboundconservative} in the main body) ]\label{thm:app_upperboundconservative}
	Suppose the true parameterization and over-parameterization are defined as in \eqref{eq:app_trueparam} and \eqref{eq:app_overparam}. Then, $\guarunsafe\subseteq \unsafeset$ and $\guarsafe\subseteq \safeset$.
\end{theorem}
\begin{proof}
	Note that \eqref{eq:app_trueparam} is equivalent to $\Big(\bigvee_{i=1}^{\bar N} \big(g_s(\state, \theta_i) \le 0\big)\Big)$, where $\theta_{\numbox+1}, \ldots, \theta_{\bar N}$ are constrained to satisfy $\{ \state \mid g_s(\state, \theta_i) \le 0\} = \emptyset, i = \numbox+1, \ldots, \bar N$. Thus, the true $\theta$ is equivalent to adding additional constraints on a loosened parameterization (the over-parameterization). Let $\hat\feas$ be the feasible set of Problem \ref{prob:parametric_feasibility_program} with $\theta$ loosened as above, i.e. $\feas = \hat\feas \cap \{\theta\mid \{ \state \mid g_s(\state, \theta_i) \le 0\} = \emptyset, i = \numbox+1, \ldots, \bar N \}$. Via Lemma \ref{lem:contain}, $\feas \subseteq \hat\feas$; thus, $I_{\neg s}(\hat\feas) \subseteq I_{\neg s}(\feas)\subseteq \unsafeset$, where the last set containment follows from Theorem \ref{thm:knownconservative}. Vice versa, $I_{s}(\hat\feas) \subseteq I_{s}(\feas)\subseteq \safeset$, where again the last set containment follows from Theorem \ref{thm:knownconservative}.
\end{proof}

\begin{theorem}[Conservativeness: Under-parameterization (Theorem \ref{thm:lowerboundconservative} in the main body) ]\label{thm:app_lowerboundconservative}
Suppose the true parameterization and under-parameterization are defined as in \eqref{eq:app_trueparam} and \eqref{eq:app_underparam}. Furthermore, assume that we incrementally grow the parameterization as described in Section \ref{sec:incremental}. Then, the following are true:
\begin{enumerate}[leftmargin=*]
	\item $\guarunsafe$ and $\guarsafe$ are not guaranteed to be contained in $\unsafeset$ (unsafe set) and $\safeset$ (safe set), respectively.
	\item Each recovered simple unsafe set $\unsafeset(\theta_i)$, $i=1,\ldots, \underline N$, for any $\theta_1, \ldots, \theta_{\underline N} \in \feas$, touches the true unsafe set (there are no spurious simple unsafe sets): for $i = 1, \ldots, \underline N$, for $\theta_1, \ldots, \theta_{\underline N} \in \feas$, $\unsafeset(\theta_i) \cap \unsafeset \ne \emptyset$ ($\underline N$ is as defined in Section \ref{sec:incremental}).
\end{enumerate}
\end{theorem}
\begin{proof}
\begin{enumerate}[leftmargin=*]
	\item We first formally prove the statement with a counterexample and then follow up with logic related to the proof of Theorem \ref{thm:app_upperboundconservative}.

	Consider the example in Fig. \ref{fig:counterexample}, where the parameterization is chosen as a single axis-aligned box $[I_{2\times 2}, -I_{2\times 2}]^\top \state \le \theta$ but $\unsafeset$ is only representable with at least two boxes. Suppose demonstrations are provided which imply that $(a_l, b_l)$ and $(a_u, b_u)$ are unsafe; then AABB$(\{(a_l, b_l), (a_u, b_u)\}) \not\subseteq \unsafeset$ is implied unsafe.
	\begin{figure}
	\centering
		\includegraphics[width=0.5\linewidth]{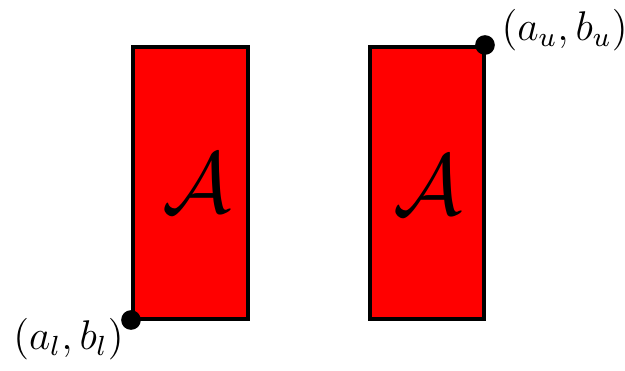}
\caption{\small Counterexample used in the proof of the first statement in Theorem \ref{thm:app_lowerboundconservative}.}
\label{fig:counterexample}
	\end{figure}

	Note that \eqref{eq:app_underparam} is equivalent to $\Big(\bigvee_{i=1}^{\numbox} \big(g_s(\state, \theta_i) \le 0\big)\Big)$, where $\theta_{\underline N+1}, \ldots, \theta_{\numbox}$ are constrained to satisfy $\{ \state \mid g_s(\state, \theta_i) \le 0\} = \emptyset, i = \underline N+1, \ldots, \numbox$. Thus, restricting the parameterization is equivalent to adding additional constraints on the true $\theta$. Let $\hat\feas$ be the feasible set of Problem \ref{prob:parametric_feasibility_program} with $\theta$ restricted as above, i.e. $\hat\feas = \feas \cap \{\theta\mid \{ \state \mid g_s(\state, \theta_i) \le 0\} = \emptyset, i = \underline N+1, \ldots, \numbox \}$. Via Lemma \ref{lem:contain}, $\hat\feas \subseteq \feas$; thus, $I_{\neg s}(\feas) \subseteq I_{\neg s}(\hat\feas)$. Since $I_{\neg s}(\feas)$ can equal $\unsafeset$, potentially $\guarunsafe = I_{\neg s}(\hat\feas) \cap \safeset \ne \emptyset$. Vice versa, $I_{s}(\feas) \subseteq I_{s}(\hat\feas)$, and since $I_{s}(\feas)$ can equal $\safeset$, potentially $\guarsafe = I_{s}(\hat\feas) \cap \safeset \ne \emptyset$.
	\item Assume, by contradiction, that Problem \ref{prob:parametric_feasibility_program} outputs a simple unsafe set $\unsafeset(\theta_i), i \in \{1, \ldots, \underline N\}$, which does not touch the true unsafe set: $\exists i \in \{1, \ldots, \underline N\}, \unsafeset(\theta_i) \cap \unsafeset(\theta^*) = \emptyset$. Then, $\theta_j, j\in \{1, \ldots, \underline N\} \setminus \{i\}$ would be a feasible point for Problem \ref{prob:parametric_feasibility_program} with a parametrization that contains only $\underline{N}-1$ simple sets. However, we know Problem \ref{prob:parametric_feasibility_program} with $\underline{N}-1$ simple sets is infeasible. Contradiction.
\end{enumerate}
\end{proof}

\section{Extra numerical examples}\label{sec:app_results}

\subsection{U-shape (random demonstrations)} 

In this example, we show what the performance of our method looks like with random demonstrations on the U-shape example. On the left of Fig. \ref{fig:ushape_rand}, we show that our coverage grows more slowly than for the case where demonstrations are chosen for their informativeness; furthermore, coverage for the safe set is higher and coverage for the unsafe set is lower in the random demonstration case. This is because by using random demonstrations, we cover a good deal of $\safeset$, so $\guarsafe$ becomes larger; on the other hand, many of these safe demonstrations may not come in contact with the constraint, so there are relatively few unsafe trajectories that can be sampled, so $\guarunsafe$ is not as large. In the center of Fig. \ref{fig:ushape_rand}, we show that the accuracy of our method doesn't change much, though the relative performance of the NN gets worse for classifying safe states; this is because the accuracy for the NN is now being evaluated on a larger region since $\guarsafe$ is larger due to more demonstrations. As in previous examples, the NN error bars are generated by training the NN ten times with initializations using different random seeds. On the right of Fig. \ref{fig:ushape_rand}, we display a feasible $\unsafeset(\theta)$ recovered by solving a multi-box variant of Problem \ref{prob:parametric_polytope_program}. With more demonstrations, the gap between $\unsafeset(\theta)$ and the true unsafe set $\unsafeset$ will continue to shrink.

The main takeaways from this experiment are: 1) when demonstrations are not informative (in the sense that they do not interact with the constraint), it can take many demonstrations to learn the unsafe set (this holds for any constraint recovery method), and 2) our accuracy remains just as high as for the case with specifically chosen demonstrations and is not much affected by the coverage.

\begin{figure}
	\centering
		\includegraphics[width=\linewidth]{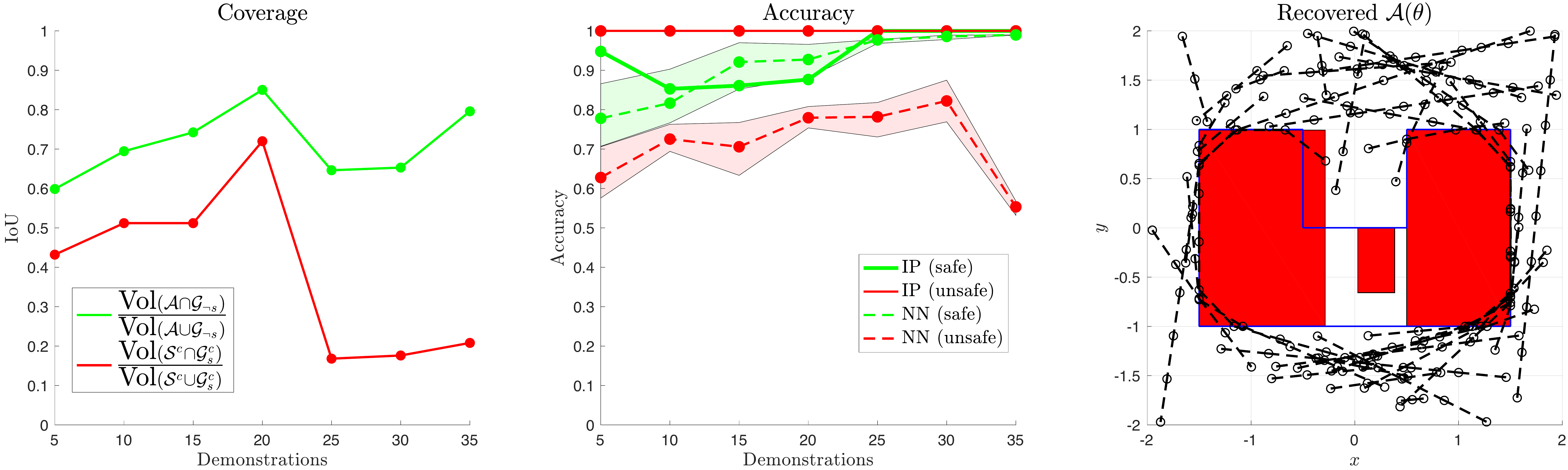}
\caption{\small U-shape performance with random demonstrations. \textbf{Left}: Coverage of $\unsafeset$ and $\safeset$. \textbf{Center}: Classification accuracy. \textbf{Right}: A recovered feasible $\unsafeset(\theta)$, overlaid with demonstrations, and the true unsafe set $\unsafeset$ is outlined in blue. }
\label{fig:ushape_rand}
	\end{figure}

\section{Experimental details}\label{sec:app_experimental}

For all neural network baseline results in every experiment, the network is trained with weights initialized using ten different random seeds, and the resulting performance range (displayed as a shaded region) and average performance over the ten random seeds are plotted in the figures.

\subsection{Unknown parameterizations}

We emphasize that for all examples with unknown parameterization, by following the incremental procedure detailed in Section \ref{sec:incremental}, we are finding the minimum number of boxes required to represent the data; in other words, we are always operating with the minimal feasible parameterization.

\textbf{U-shape and infinite boxes}:

\begin{itemize}[leftmargin=*]
	\item For both experiments, the system dynamics are $\state_{t+1} \doteq [\chi_{t+1}, y_{t+1}]^\top = [\chi_t, y_t]^\top + [u_t^\chi, u_t^y]^\top$. The U-shape experiment uses control constraints $\Vert[u_t^\chi, u_t^y]\Vert_2 \le 0.5$, while the infinite-box experiment uses control constraints $\Vert[u_t^\chi, u_t^y]\Vert_2 \le 1$.
	\item For both experiments, the cost function is $c(\trajx, \traju) = \sum_{i=1}^{T-1} \Vert \state_{t+1} - \state_t\Vert_2^2$.
	\item Since the cost function has optimal substructure, 100000 unsafe trajectories for each sub-trajectory are sampled. The dataset is downsampled to 50 unsafe trajectories for each sub-trajectory, which are to be fed into the multi-box variant of Problem \ref{prob:parametric_polytope_program}.
	\item For both experiments, the initial parameter set is restricted to $[-5, -5, -3, -3]^\top \le \theta_i \le [8, 8, 3, 3]^\top$, for each $\theta_i$ (the parameter for box $i$). For the infinite-box experiment, each box is restricted to be at least $1.25 \times 1.25$ in width/height.
	\item Sampling time is around 15 seconds per demonstration (for the U-shape experiment) and 10 seconds per demonstration (for the infinite-box experiment). Computation time for solving Problem \ref{prob:parametric_polytope_program} is around 40 seconds (for the U-shape experiment) and 15-20 seconds (for the infinite-box experiment).
	\item The same data is used for training the neural network (7800 trajectories total for the U-shape case, 2000 trajectories for the infinite-box case). The neural network architecture used for this example is a fully connected (FC) layer, $2\times 10$ $\rightarrow$ LSTM, $10\times 10 \rightarrow$ FC $10\times 1$ (the recurrent layer is used since we have variable length trajectories as training input). The network is trained using Adam.
\end{itemize}

\textbf{U-shape with random demonstrations}:

\begin{itemize}[leftmargin=*]
	\item The system dynamics are $\state_{t+1} \doteq [\chi_{t+1}, y_{t+1}]^\top = [\chi_t, y_t]^\top + [u_t^\chi, u_t^y]^\top$ with control constraints $\Vert[u_t^\chi, u_t^y]\Vert_2 \le 0.5$.
	\item The cost function is $c(\trajx, \traju) = \sum_{i=1}^{T-1} \Vert \state_{t+1} - \state_t\Vert_2^2$.
	\item Demonstrations are generated for 35 pairs of start/goal states sampled uniformly at random over $(\chi,y) \in [-2, 2] \times [-2, 2]$, rejecting any start/goal states that lie in $\unsafeset$.
	\item Since the cost function has optimal substructure, 10000 unsafe trajectories for each sub-trajectory are sampled. The dataset is downsampled to 25 unsafe trajectories for each sub-trajectory, which are to be fed into the multi-box variant of Problem \ref{prob:parametric_polytope_program}.
	\item The initial parameter set is restricted to $[-5, -5, -3, -3]^\top \le \theta_i \le [8, 8, 3, 3]^\top$, for each $\theta_i$ (the parameter for box $i$).
	\item Sampling time is around 2 minutes total. Computation time for solving the multi-box variant of Problem \ref{prob:parametric_polytope_program} is around 90 seconds.
	\item The same data is used for training the neural network (10100 trajectories total). The neural network architecture used for this example is a fully connected (FC) layer, $2\times 10$ $\rightarrow$ LSTM, $10\times 10 \rightarrow$ FC $10\times 1$. The network is trained using Adam.
\end{itemize}

\subsection{High-dimensional examples}

\textbf{7-DOF arm, optimal/suboptimal demonstrations}

\begin{itemize}[leftmargin=*]
	\item The system dynamics are $\doteq \theta_{t+1}^i = \theta_{t}^i + \control_t^i$, $i=1, \ldots, 7$, with control constraints $-2 \le \control_t^i \le 2$, $i=1, \ldots, 7$, where the state is $\state = [\theta^1, \ldots, \theta^7]$.
	\item The cost function is $c(\trajx, \traju) = \sum_{i=1}^{T-1} \Vert \state_{t+1} - \state_t\Vert_2^2$. Note that the generate demonstrations (displayed in Fig. \ref{fig:arm}) push up against the position constraint, since the trajectory minimizing joint-space path length without the position constraint is an arc that exceeds the bounds of the position constraint; the position constraint ends up increasing the cost by truncating that arc.
	\item The true safe set is $(x, y, z, \alpha, \beta, \gamma) \in [-0.51, 0.51] \times [-0.3, 1.1] \times [-0.51, 0.51] \times [-\pi, \pi] \times [-\pi/120, \pi/120] \times [-\pi/120, \pi/120]$ for the optimal case and the true safe set is $(x, y, z, \alpha, \beta, \gamma) \in [-0.57, 0.47] \times [-0.10, 1.17] \times [-0.56,    0.56]\times [-\pi, \pi]\times [-0.12, 0.12] \times [-0.125, 0.125]$ for the suboptimal case.
	\item Since the cost function has optimal substructure, 250000 unsafe trajectories for each sub-trajectory are sampled. For the suboptimal case, the continuous-time demonstrations are time-discretized down to $T = 10$ time-steps. The dataset is downsampled to 500 unsafe trajectories for each sub-trajectory, which are to be fed into Problem \ref{prob:parametric_polytope_program}.
	\item For the optimal case, the demonstrations are obtained by solving trajectory optimization problems solved with the IPOPT solver \cite{ipopt}. For the suboptimal case, the demonstrations are recorded in a virtual reality (VR) environment displayed in Fig. \ref{fig:vive}.
	\item The initial parameter set is restricted to $[-1.5, -1.5, -1.5, -\pi, -\pi, -\pi]^\top \le [x, y, z, \alpha, \beta, \gamma]^\top \le [1.5, 1.5, 1.5, \pi, \pi, \pi]^\top$.
	\item Sampling time is 12.5 minutes total for the optimal case and 9 minutes total for the suboptimal case. Computation time for solving Problem \ref{prob:parametric_feasibility_program} is around 2 seconds for both the optimal/suboptimal case.
	\item The same data is used for training the neural network (70000 trajectories total for the optimal case, 49900 trajectories total for the suboptimal case). The neural network architecture used for this example is a fully connected (FC) layer, $3\times 20$ $\rightarrow$ LSTM, $20\times 20 \rightarrow$ FC $20\times 1$. The network is trained using Adam.
\end{itemize}

\begin{figure}
	\vspace{-15pt}
	\centering
	\begin{subfigure}[b]{0.45\textwidth}
         \centering
         \includegraphics[width=\linewidth]{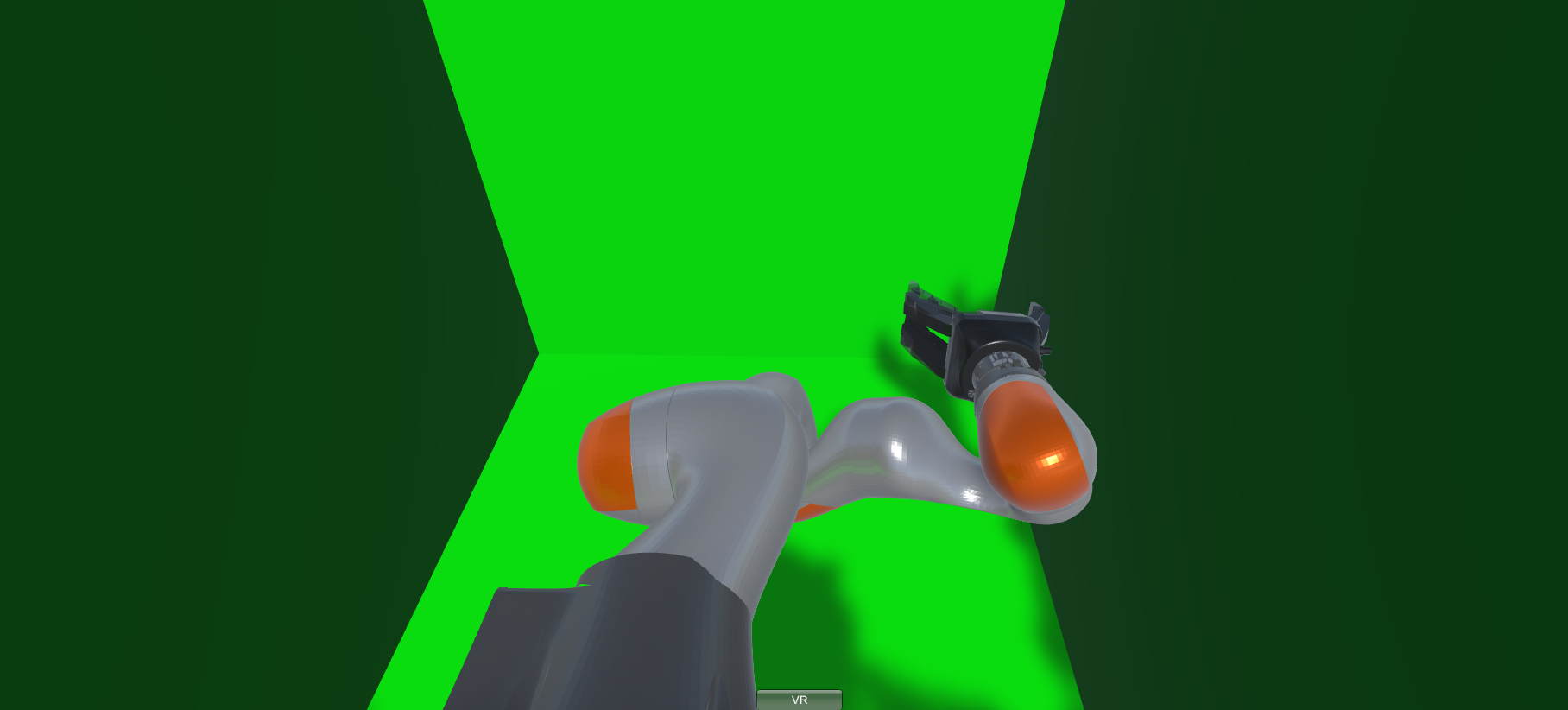}
         \vspace{-10pt}
     \end{subfigure}
     \begin{subfigure}[b]{0.3\textwidth}
         \centering
         \includegraphics[width=\linewidth]{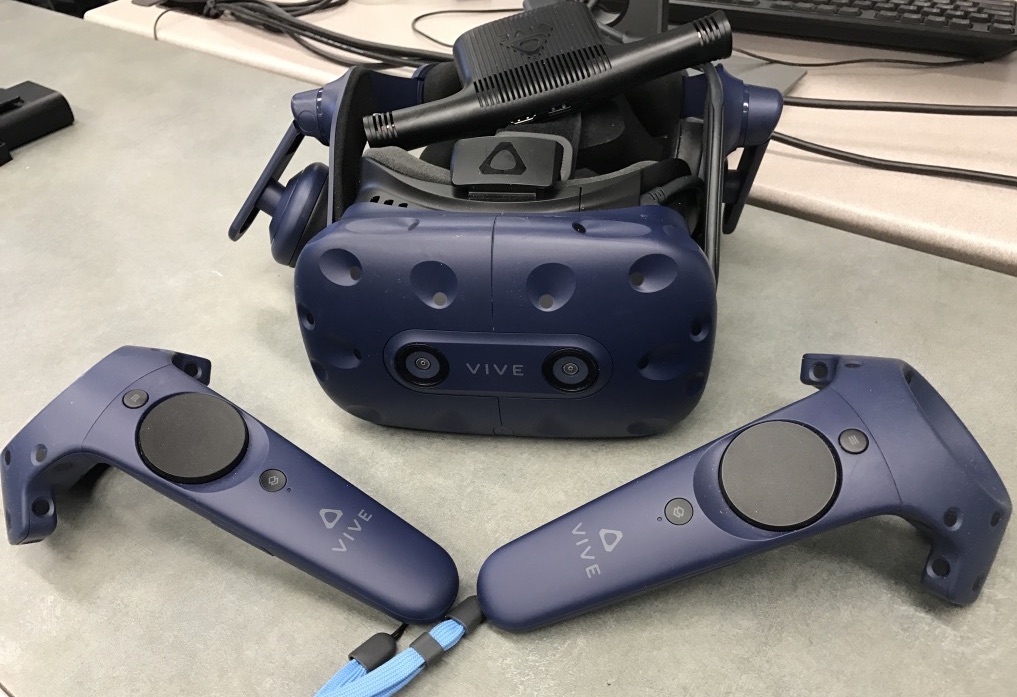}
     \end{subfigure}
     \caption{\small VR setup. \textbf{Left}: VR environment as viewed from the Vive headset. The green box represents the position constraints on the end effector. The end effector is commanded to move by dragging it with the HTC Vive controllers (\textbf{right}).}
     \vspace{-15pt}
     \label{fig:vive}
\end{figure}
\vspace{80pt}
\textbf{12D quadrotor example}

\begin{itemize}[leftmargin=*]
	\item The system dynamics \cite{quad_kth} are 
\begin{equation}
	\hspace{-5pt}\begin{bmatrix} \dot\chi \\ \dot y \\ \dot z \\ \dot\alpha \\ \dot\beta \\ \dot\gamma \\ \ddot \chi \\ \ddot y \\ \ddot z \\ \ddot \alpha \\ \ddot \beta \\ \ddot \gamma \end{bmatrix} = \begin{bmatrix} \dot\chi \\ \dot y \\ \dot z \\ \dot\beta \frac{\sin(\gamma)}{\cos(\beta)} + \dot\gamma \frac{\cos(\gamma)}{\cos(\beta)} \\ \beta \cos(\gamma) - \dot\gamma \sin(\gamma) \\ \dot\alpha + \dot\beta\sin(\gamma)\tan(\beta)+\dot\gamma\cos(\gamma)\tan(\beta) \\ -\frac{1}{m}[\sin(\gamma)\sin(\alpha) + \cos(\gamma)\cos(\alpha)\sin(\beta)]u_1 \\ -\frac{1}{m}[\cos(\alpha)\sin(\gamma) - \cos(\gamma)\sin(\alpha)\sin(\beta)]u_1 \\ g-\frac{1}{m}[\cos(\gamma)\cos(\beta)]u_1 \\ \frac{I_y-I_z}{I_x} \dot\beta \dot\gamma + \frac{1}{I_x}u_2 \\ \frac{I_z-I_x}{I_y} \dot\alpha \dot\gamma + \frac{1}{I_y}u_3\\ \frac{I_x-I_y}{I_z} \dot\alpha \dot\beta + \frac{1}{I_z}u_4\end{bmatrix},
\end{equation}
	with control constraints $[0, -0.02, -0.02, -0.02]^\top \le u_t \le [mg, 0.02, 0.02, 0.02]^\top$. For our purposes, we convert the dynamics to discrete time by performing forward Euler integration with discretization time $\delta t = 0.4$ seconds. The state is $x = [\chi, y, z, \alpha, \beta, \gamma, \dot x, \dot y, \dot z, \dot \alpha, \dot \beta, \dot \gamma]^\top$, and the constants are $g = -9.81 \textrm{m}/\textrm{s}^2$, $m=1$kg, $I_x = 0.5\textrm{kg}\cdot\textrm{m}^2$, $I_y = 0.1\textrm{kg}\cdot\textrm{m}^2$, and $I_z = 0.3\textrm{kg}\cdot\textrm{m}^2$.
	\item The known unsafe set in $(\chi,y,z)$ is $(\chi,y,z) \notin [-0.5, 0.5]\times [-0.5, 0.5] \times [-0.5, 0.5]$.
	\item The true safe set in $(\dot \alpha, \dot \beta, \dot\gamma)$ is $(\dot\alpha, \dot\beta, \dot\gamma) \in [-0.006, 0.006]^3$.
	\item The cost function is $c(\trajx, \traju) = \sum_{i=1}^{T-1} \Vert [\chi_{i+1}, y_{i+1}, z_{i+1}, \dot \alpha_{i+1}, \dot \beta_{i+1}, \dot \gamma_{i+1}]^\top - [\chi_{i}, y_{i}, z_{i}, \dot \alpha_{i}, \dot \beta_{i}, \dot \gamma_{i}]^\top \Vert_2$ (penalizing acceleration and path length).
	\item The demonstrations are obtained by solving trajectory optimization problems solved with the IPOPT solver \cite{ipopt}.
	\item Since the cost function has optimal substructure, 10000 unsafe trajectories for each sub-trajectory are sampled. The dataset is downsampled to 500 unsafe trajectories for each sub-trajectory, which are to be fed into Problem \ref{prob:parametric_polytope_program}.
	\item The initial parameter set is restricted to $[-\pi/2, -\pi/2, -\pi/2]^\top \le [\dot\alpha, \dot\beta, \dot\gamma]^\top \le [\pi/2, \pi/2, \pi/2]^\top$.
	\item Sampling time is 8.5 minutes total for the optimal case and 9 minutes total for the suboptimal case. Computation time for solving Problem \ref{prob:parametric_feasibility_program} is 12 seconds.
	\item The same data is used for training the neural network (30000 trajectories total). The neural network architecture used for this example is a fully connected (FC) layer, $6\times 36$ $\rightarrow$ LSTM, $36\times 42 \rightarrow$ FC $42\times 1$. The network is trained using Adam.
\end{itemize}

\subsection{Black-box system dynamics}

\textbf{Pushing example}

\begin{itemize}[leftmargin=*]
	\item The cost function is $c(\trajx, \traju) = \sum_{i=1}^{T-1} \Vert \state_{t+1} - \state_t\Vert_2^2$. The two demonstrations are manually generated and are not exactly optimal.
	\item 1000 unsafe trajectories for each demonstrations are sampled. 
	\item The initial parameter set is restricted to $[-5, -5, -3, -3]^\top \le \theta_i \le [8, 8, 3, 3]^\top$.
	\item Sampling time is 2 hours for each demonstration (using the simulator is slower than using the closed form dynamics). Computation time for solving Problem \ref{prob:parametric_feasibility_program} is around 1 second.
	\item Demonstrations are time-discretized to 40 simulator timesteps when input to Problem \ref{prob:parametric_polytope_program}.
	\item The same data is used for training the neural network (2700 trajectories total). The neural network architecture used for this example is a fully connected (FC) layer, $8\times 10$ $\rightarrow$ FC, $10\times 10 \rightarrow$ FC $10\times 1$. No recurrent layer is used this time since all trajectories are of the same length (no sub-trajectories were sampled this time due to speed). The network is trained using Adam.
\end{itemize}

\vspace{10pt}
\section{Summary of frequently used notation}\label{sec:app_notation}

\begin{figure}[H]
\small
\begin{tabular}{ | c || c | } 
\hline
Meaning & Notation\\ 
\hline
\hline
State, state space & $\state$, $\statespace$ \\ 
\hline
Control, control space & $\control$, $\controlset$\\
\hline
State/control trajectory & $\trajx$, $\traju$\\
\hline
Constraint state, constraint space & $\cstate$, $\constraintspace$\\
\hline
Safe set, unsafe set & $\safeset$, $\unsafeset$ \\
\hline
Parameterized safe set & $\safeset(\theta) = \{\cstate\mid g(\cstate, \theta) > 0\}$\\
\hline
Parameterized unsafe set & $\unsafeset(\theta) = \{\cstate\mid g(\cstate, \theta) \le 0\}$\\
\hline
Safe demonstration $j$ & $\traj_{s_j}^*$\\
\hline
Sampled unsafe trajectory $k$ & $\traj_{\neg s_k}$\\
\hline
Guaranteed safe set & $\guarsafe$\\
\hline
Guaranteed unsafe set & $\guarunsafe$\\
\hline
\end{tabular}
\centering
\captionof{table}{Notation.}
\end{figure}

\end{document}